\documentclass{article}

\usepackage[nonatbib,preprint]{neurips_2026}
\usepackage[numbers,sort&compress]{natbib}

\usepackage[utf8]{inputenc}
\usepackage[T1]{fontenc}
\usepackage{amsmath,amssymb,amsfonts}
\usepackage{mathtools}
\usepackage{algorithm}
\usepackage{algpseudocode}
\usepackage{graphicx}
\usepackage{subcaption}
\usepackage{booktabs}
\usepackage{hyperref}
\usepackage{cleveref}
\usepackage{xcolor}
\usepackage{enumitem}
\usepackage{amsfonts}       
\usepackage{nicefrac}       
\usepackage{microtype}
\usepackage{float}
\usepackage{microtype}      
\usepackage{xcolor}         
\usepackage{multirow}
\usepackage{url}

\usepackage[T1]{fontenc}
\usepackage[utf8]{inputenc}
\usepackage[warn]{textcomp}
\usepackage{amsthm}

\usepackage{siunitx}
\usepackage{enumitem}
\usepackage{xspace}
\usepackage{xfrac}
\usepackage{bm}
\usepackage{array}
\usepackage{adjustbox}
\usepackage{caption}
\usepackage{subcaption}
\usepackage{multirow}
\usepackage{graphicx}
\usepackage{amsmath}
\usepackage{amssymb}
\usepackage{booktabs}
\usepackage{mathtools}
\usepackage{algorithm}
\usepackage{algpseudocode}

\usepackage{amsmath, amsthm, amssymb}

\newtheorem{proposition}{Proposition}
\newtheorem{theorem}{Theorem}[section]

\newtheorem{lemma}[theorem]{Lemma}

\newcommand{\expect}{\mathbb{E}}
\newcommand{\real}{\mathbb{R}}
\newcommand{\norm}[1]{\lVert#1\rVert}
\newcommand{\vect}[1]{\bm{#1}}
\newcommand{\matt}[1]{\bm{\MakeUppercase{#1}}}

\newcommand{\fig}[1]{Fig.~\ref{fig:#1}}
\newcommand{\Figure}[1]{Figure~\ref{fig:#1}}

\newcommand{\Table}[1]{Table~\ref{tab:#1}}
\newcommand{\eqn}[1]{Eq.~\ref{eq:#1}}
\newcommand{\sect}[1]{Sec.~\ref{sec:#1}}
\newcommand{\Section}[1]{Section~\ref{sec:#1}}
\newcommand{\algo}[1]{Alg.~\ref{alg:#1}}
\newcommand{\Algo}[1]{Algorithm~\ref{alg:#1}}
\newcommand{\titlecaption}[2]{\caption[#1]{\textbf{#1.}~#2}}
\newcommand{\titlecaptionof}[3]{\captionof{#1}{\textbf{#2.}\xspace#3}}

\makeatletter
\DeclareRobustCommand\onedot{\futurelet\@let@token\@onedot}
\def\@onedot{\ifx\@let@token.\else.\null\fi\xspace}

\makeatother

\newcommand{\inlinesection}[1]{\vspace{0.05cm}\noindent\textbf{#1}}

\usepackage{amsmath,amssymb,bm}
\usepackage{tikz}
\usetikzlibrary{arrows.meta,positioning,calc,shapes.geometric,shapes.misc,
                decorations.pathreplacing,decorations.pathmorphing,
                fit,backgrounds,3d,patterns,matrix,shadows.blur}
\usepackage{xcolor}

\definecolor{octblue}{HTML}{1F6FEB}
\definecolor{octteal}{HTML}{138086}
\definecolor{octamber}{HTML}{D97706}
\definecolor{octrose}{HTML}{C2185B}
\definecolor{octplum}{HTML}{6B21A8}
\definecolor{octslate}{HTML}{334155}
\definecolor{octmuted}{HTML}{64748B}
\definecolor{octpaper}{HTML}{F8FAFC}
\definecolor{octsurface}{HTML}{E2E8F0}
\definecolor{octaccent}{HTML}{0EA5E9}
\definecolor{qjlgreen}{HTML}{059669}
\definecolor{codebook}{HTML}{7C3AED}

\newcommand{\capt}[1]{{\scriptsize\itshape\color{octmuted}#1}}

\newcommand{\ours}{OCTOPUS\xspace}

\makeatletter
\newif\ifshowchecklist
\newif\ifpreprint
\if@preprint
  \showchecklistfalse
  \preprinttrue
\else
  \showchecklisttrue
  \preprintfalse
\fi
\makeatother

\begin{document}

\title{\ours: Optimized KV Cache for Transformers via Octahedral Parametrization Under optimal Squared error quantization}
\author{Mark Boss \\
  Stability AI \and Vikram Voleti \\
  Stability AI \and Simon Donné \\
  Stability AI \and Shimon Vainer \\
  Stability AI}

\date{\today}

\maketitle

\begin{abstract}
The key-value (KV) cache dominates memory bandwidth and footprint in
long-context autoregressive inference. Recent rotation-preconditioned codecs (TurboQuant,
PolarQuant) show that a structured random rotation followed by a per-coordinate scalar quantizer matched to an analytically tractable marginal is a near-optimal recipe for KV compression.
\ours\ advances this paradigm through joint quantization of rotated coordinate \emph{triplets}. Each triplet's direction is mapped to a square via an octahedral parameterization, and the two resulting coordinates and the triplet norm are Lloyd-Max quantized against implementation-matched marginals.
Optimizing the per-triplet squared error gives a strictly non-uniform bit allocation depending only on the total dimensionality of the keys.
We find the finite-dimensional quality optimum with sweeps to be constant on every real decoder we test. The codec is data-oblivious, online, and deterministic given a seed.
Across text, video, and audio, \ours matches or beats every prior rotation codec at every reported bit width and metric, with a lead that grows as bits drop for extreme compression.
Furthermore, a fused Triton implementation reconstructs keys on the fly without materializing the uncompressed key, so the codec adds no decode-time bandwidth or latency over the existing dequantization. %
\ifpreprint
Project Page: \url{https://octopus-quant.github.io/}
\fi
\end{abstract}

\section{Introduction}
\label{sec:intro}

Long-context autoregressive inference, such as in large language models (LLM)~\cite{dubey2024llama3}, causal video generation models~\cite{yin2025causvid,zhu2026causal}, or audio generation models~\cite{qiu2024efficient}, is dominated by reading the
key-value (KV) cache from high-bandwidth memory at every decoding
step~\citep{fu2024data,liu2024lostinmiddle}. KV compression is therefore
the primary target for both latency and batch-size optimization, and prior works address it through token eviction~\citep{xiao2023streamingllm,zhang2024h2o,li2024snapkv},
per-channel scalar quantization with
residuals~\citep{hooper2024kvquant,liu2024kivi,kang2024gear}, and more
recently rotation-preconditioned quantization codecs
\citep{zandieh2025turboquant,han2025polarquant,zandieh2024qjl}.

Rotation-based codecs depend on a
structured random orthogonal $\matt{R}$ (typically a sign-flipped
Walsh-Hadamard transform due to efficiency~\citep{chee2023quip}) to make the marginal of every rotated key
coordinate isotropic and \emph{analytically} known. A 1-D Lloyd-Max
quantizer~\citep{lloyd1982least,max1960quantizing} matched to that
marginal is then near-optimal at matched bit width. In this way, TurboQuant~\citep{zandieh2025turboquant} gets a
symmetric Beta marginal, PolarQuant~\citep{han2025polarquant} does the analogous
construction on recursive polar angles, and the QJL 1-bit residual
makes the dot product unbiased at near-zero memory cost~\citep{zandieh2024qjl}. All three
quantize one coordinate (or one angle) at a time. \ours instead
quantizes coordinate-\emph{triplets} jointly.

Two observations motivate \ours. \emph{First}, the rotation pre-conditioning evenly spreads entropy across the coordinates: the norm of a small sub-block carries asymptotically less entropy with rising channel count.
We show that a codec that quantizes sub-block norm and direction separately, with non-uniform bit allocation between them, beats the per-coordinate quantizers at matched rate.
\emph{Second}, the octahedral map from computer graphics~\citep{engelhardt2008octahedron,cigolle2014survey} is an equal-area parameterization of $S^2$ that can encode a unit 3-vector
as two scalars on $[-1,1]^2$ in $\mathcal{O}(1)$ arithmetic operations,
with piecewise-linear encode/decode and a near-uniform Jacobian that
makes 1-D Lloyd-Max on the induced marginals a close approximation
to true 2-sphere distortion.
Therefore, \ours splits the pre-conditioned signal into triplets, and Lloyd-Max-quantizes the triplet norm and the octahedrally-mapped triplet direction coordinates with non-uniform bit depth.
There is no data-dependent calibration or per-vector scale: codebooks depend only on $d$ and the bit budget. Our contributions are:

\begin{itemize}[leftmargin=1.25em,itemsep=0.15em]
\item \textbf{Octahedral triplet direction quantizer} as a KV cache
primitive, with implementation-matched norm and direction marginals. The compress-decode pipeline is implemented as fused Triton
kernels~\citep{tillet2019triton,dao2022flashattention,shah2024flashattention3}
that reconstruct keys on the fly from packed bit indices and never needs to materialize the full key tensor.
\item \textbf{An MSE-optimal non-uniform bit split.} A Lagrangian on
the per-triplet squared error yields a finite-dimensional stationarity
condition that supports the implemented $(b{+}1,b{-}1)$ split at
$d{=}128$.
\item \textbf{Optional 1-bit QJL residual} (\ours-QJL) that drives the
seed-averaged dot-product bias to zero at the cost of one sign bit per
rotated coordinate.
\item \textbf{Generalization beyond LLMs.} Prior rotation-preconditioned
KV codecs are evaluated only on language models, but the construction is
agnostic to the source of the keys: any autoregressive transformer with
attention should benefit. We confirm this empirically. \ours\ is the
best rotation-based codec at matched bit widths $K{=}V\in\{4,3,2\}$ in
long-context language modeling
(Qwen2.5-7B-Instruct-1M~\citep{dubey2024llama3}), chunk-wise video
diffusion (CausVid~\citep{yin2025causvid}), frame-wise causal video
forcing~\citep{zhu2026causal}, and next-scale autoregressive
audio~\citep{qiu2024efficient}, with larger gaps at lower bit budgets.
\end{itemize}

\Section{related} situates \ours in the
literature; \Section{method} develops the codec; and
\Section{results} reports end-to-end numbers across the four
modalities.
Appropriate proofs are found in the Appendix.
\section{Related Work}
\label{sec:related}

\inlinesection{KV-cache compression.} Token
eviction~\citep{xiao2023streamingllm,zhang2024h2o,li2024snapkv,liu2024scissorhands,cai2024pyramidkv}
keeps only tokens that are likely to contribute to future attention.
Per-channel scalar quantization with per-token residuals attacks the
distribution of individual key
coordinates~\citep{hooper2024kvquant,liu2024kivi,kang2024gear,yang2024notoken,zhang2024kv1bit,dong2024qaq,yue2024wkvquant}.
Sparse coding~\citep{kim2024lexico} trades a bigger code table for
ultra-low rates. Rotation-preconditioned
codecs~\citep{zandieh2025turboquant,han2025polarquant,wu2025polarquant,zandieh2024qjl,ashkboos2024quarot,su2025rotatekv,han2025balancekv}
project keys by a data-oblivious random orthogonal operator so that the
marginals fed to the quantizer are analytically known; \ours belongs
to this last family.

\inlinesection{Rotation-preconditioned quantization.}
TurboQuant~\citep{zandieh2025turboquant} proves that a random orthogonal
rotation makes every coordinate of a unit vector marginally
symmetric-Beta on $[-1,1]$, so the MSE-optimal 1-D
Lloyd-Max~\citep{lloyd1982least,max1960quantizing} codebook depends
only on $(d,b)$ and lands within a small constant of the
Zador-Gersho~\citep{zador1964development,gersho1979asymptotically}
bound. The structured Walsh-Hadamard transform with random sign flips
is the standard fast
preconditioner~\citep{chee2023quip,ashkboos2024quarot,su2025rotatekv}.
PolarQuant~\citep{han2025polarquant} parameterises the rotated direction
recursively in polar coordinates instead. \ours reuses the
Walsh-Hadamard rotation but quantizes blocks of three rotated
coordinates jointly via an octahedral direction+norm split, which we
show gives strictly lower MSE at matched bit rate.

\inlinesection{Unit-direction encodings and unbiased estimators.}
Octahedral and related equal-area parameterizations of $S^2$ are the
de-facto compact direction encoding in real-time
rendering~\citep{engelhardt2008octahedron,cigolle2014survey}; to our knowledge \ours is the
first use of the octahedral map as a direction quantizer in transformer
decoding. Orthogonal to MSE-optimal codecs,
QJL~\citep{zandieh2024qjl} shows that a 1-bit
Johnson-Lindenstrauss sketch gives an unbiased inner-product estimator
at essentially zero memory; we compose it with \ours under the tag
\textsc{\ours-QJL}. We borrow only the rotation idea from the broader
quantization literature on
weights~\citep{frantar2022gptq,lin2024awq,chee2023quip} and
weight$+$activation
quantization~\citep{dettmers2022gpt3int8,xiao2023smoothquant,zhao2024atom,ashkboos2024quarot};
the codec, bit allocation and codebooks are specific to the KV cache
and online by construction. Fused attention
kernels~\citep{dao2022flashattention,shah2024flashattention3} keep our
reconstruction in registers.

\section{Method}
\label{sec:method}

\begin{figure}[t]
    \centering
    \resizebox{0.95\linewidth}{!}{\input{images/overview}}
    \titlecaption{The \ours encode pipeline}{%
      Stages~1--5 (top) realise the rotation and triplet decomposition of
      \sect{method:turboquant}--\ref{sec:method:triplets}: a key
      $\vect{k}$ is normalised (Eq.~\ref{eq:split}), preconditioned by a
      sign-flipped Walsh-Hadamard rotation (Eq.~\ref{eq:rotation}),
      cut into $n_\mathrm{tri}=\lceil d/3\rceil$ triplets, and decomposed
      into a triplet norm $\rho_i$ and a unit direction $\bm{n}_i\in S^2$
      (Sec.~\ref{sec:method:triplets}). Stage~6 (middle) maps each direction
      onto $[-1,1]^2$ via the octahedral fold
      (Eq.~\ref{eq:oct-encode}--\ref{eq:oct-decode}); the analytic
      triplet-norm marginal (Eq.~\ref{eq:rho-pdf}) and empirical oct-coordinate
      marginal (Eq.~\ref{eq:oct-marginal}) drive the Lagrangian bit allocation
      of \sect{method:bitsplit}, whose finite-dimensional stationarity
      condition (Eq.~\ref{eq:bitsplit-opt}) motivates the implemented
      $(b{+}1,b{-}1)$ split. Stage~7--8 (bottom) emit the compressed state
      $\mathcal{S}(\vect{k})=(\gamma,\mathcal{I}_\mathrm{dir},
      \mathcal{I}_\mathrm{nrm})$ via the Lloyd-Max codebooks of
      \sect{method:codebooks} followed by the local $3{\times}3$ joint
      rounding of \sect{method:rounding}. The optional QJL side-car (stage~9,
      \sect{method:score}) attaches a 1-bit sign sketch of the
      rotated-frame residual for an ideal-model unbiased dot-product estimator. The
      bottom strip summarises the fused decode kernel
      (\algo{decode}, App.~\ref{app:algos}).}%
    \label{fig:overview}
\end{figure}

\Figure{overview} previews the pipeline. Given a key
$\vect{k}\in\real^{d}$, the \ours encoder produces a compressed state
$\mathcal{S}(\vect{k})=(\gamma,\mathcal{I}_{\mathrm{dir}},
\mathcal{I}_{\mathrm{nrm}})$: the global norm, a packed stream of
octahedral-coordinate indices, and a packed stream of triplet-norm
indices. The decoder reconstructs a lossy $\hat{\vect{k}}$ inside
attention and never materialises $\hat{\matt{K}}$. We assume $d$ is a
power of two, as required by the Walsh-Hadamard transform.

\subsection{Rotation preconditioning}
\label{sec:method:turboquant}

We split each nonzero $\vect{k}$ into magnitude
$\gamma \in \mathbb{R}^{+}$ and direction
$\vect{\tilde u} \in S^{d-1}$:
\begin{equation}
\label{eq:split}
    \gamma \coloneqq \norm{\vect{k}}_2,\qquad
    \vect{\tilde u}\coloneqq \vect{k}/\gamma .
\end{equation}
The magnitude is stored as float32 (4\,B per key; $= 0.25$ bpc at
$d{=}128$), so almost the entire quantization budget goes to the unit
direction. We precondition $\vect{\tilde u}$ by a sign-flipped
Walsh-Hadamard transform: with
$\vect{s}\in\{\pm 1\}^{d}$ drawn once per attention head and $\matt{H}$
the normalised Hadamard matrix,
\begin{equation}
\label{eq:rotation}
    \matt{R} \coloneqq \matt{H}\,\mathrm{diag}(\vect{s}),\qquad
    \vect{u}\coloneqq \matt{R}\,\vect{\tilde u}\in S^{d-1}.
\end{equation}
$\matt{R}$ is orthogonal and its inverse runs in
$\mathcal{O}(d\log d)$ via an in-place butterfly. Inner products are
preserved
($\vect{q}^{\!\top}\vect{k}=\gamma\,(\matt{R}\vect{q})^{\!\top}\vect{u}$),
and each coordinate of high-dimensional $\vect{u}$ has the marginal
\begin{equation}
\label{eq:beta-marginal}
    f(u) \sim (1-u^2)^{(d-3)/2}  \qquad u\in[-1,1].
\end{equation}

\subsection{Triplet decomposition and octahedral coordinates}
\label{sec:method:triplets}

TurboQuant's MSE baseline quantizes $\vect{u}$ with a per-coordinate
Lloyd-Max~\citep[Thm.~1]{zandieh2025turboquant}. \ours instead
quantizes \emph{triplets} of rotated coordinates jointly.
We partition $\vect{u}$ into $n_{\mathrm{tri}}=\lceil d/3\rceil$
contiguous triplets $\vect{t}_i\in\real^{3}$, zero-padding the last. For each
triplet $\vect{t}_i$ we again split its norm $\rho_i \in \mathbb{R}^{+}$ from its direction $\vect{n}_i  \in S^{2}$.
When $\rho_i=0$, the implementation uses an $\epsilon$-safe divisor and
stores a placeholder direction.
\begin{lemma}[\textbf{Triplet-norm marginal}]
\label{lem:triplet-norm}
For $\vect{u}$ uniform on $S^{d-1}$, $\rho_i^{2}\sim
\mathrm{Beta}(3/2,(d-3)/2)$, so $\rho_i$ has density
\begin{equation}
\label{eq:rho-pdf}
    f_{\rho}(r)
    = \frac{2 r^{2}\,(1-r^{2})^{(d-5)/2}}{B(3/2,(d-3)/2)},
    \qquad r\in[0,1].
\end{equation}
\end{lemma}
\noindent
As $d\to\infty$ the scale $\sqrt{\expect[\rho_i^2]}=\sqrt{3/d}$
vanishes, so radial errors contribute less absolute squared error than
direction errors.

\inlinesection{Octahedral parameterization.}
We encode $\vect{n}_i\in S^{2}$ as two scalars on $[-1,1]^{2}$ via the
octahedral map~\citep{engelhardt2008octahedron,cigolle2014survey}.
With $(x,y,z)=\vect{n}_i$ and $\ell=|x|+|y|+|z|$, project to the
octahedron $\{\ell=1\}$ via $(p_x,p_y,p_z)=\vect{n}_i/\ell$, then unfold
to a square in $[-1,1]^{2}$:
\begin{equation}
\label{eq:oct-encode}
    \mathrm{Oct}(\vect{n}_i) = (\xi, \eta) \coloneqq
    \begin{cases}
        (p_x,\,p_y) & \text{if } p_z\geq 0,\\[1pt]
        \bigl(\mathrm{sign}(p_x)(1-|p_y|),\,
              \mathrm{sign}(p_y)(1-|p_x|)\bigr) & \text{if } p_z<0.
    \end{cases}
\end{equation}
The decoder inverts this: given $(\xi,\eta)\in[-1,1]^{2}$,
\begin{equation}
\label{eq:oct-decode}
    \vect{n}(\xi,\eta)
    =\frac{(\xi',\eta',\,1-|\xi|-|\eta|)}
          {\norm{(\xi',\eta',\,1-|\xi|-|\eta|)}_2},
\end{equation}
with $(\xi',\eta')=(\xi,\eta)$ if $1-|\xi|-|\eta|\geq 0$ and
$(\xi',\eta')=(\mathrm{sign}(\xi)(1-|\eta|),\mathrm{sign}(\eta)(1-|\xi|))$
otherwise. The map is a piecewise linear bijection $S^{2}\!\to[-1,1]^2$
with a constant Jacobian per octant~\citep{engelhardt2008octahedron,cigolle2014survey}.
The octahedral
fold maps to a \emph{square} code space, so per-coordinate Lloyd-Max
on $(\xi,\eta)$ closely approximates the true 2-sphere distortion, while recursive polar parameterizations~\citep{han2025polarquant} need transcendental operators and induce
$\sin^{2^{\ell-1}\!-1}(2\psi)$ angle marginals.

Under the uniform prior on $S^2$, the octahedral-coordinate marginal is
non-uniform. Writing $a=|\xi|$, the marginal induced by the implemented fold is
\begin{equation}
\label{eq:oct-marginal}
    f_{\xi}(\xi)=
    \frac{1}{\pi\sqrt{a^2+(1-a)^2}}
    \left(
      \frac{1-a}{1-2a+3a^2}
      + \frac{a}{2-4a+3a^2}
    \right),
    \qquad a=|\xi|,\;\xi\in[-1,1],
\end{equation}
and $\eta$ shares this marginal by symmetry. Rather than directly evaluate
\eqn{oct-marginal}, the implementation trains a Lloyd-Max 1-D 
codebook on empirical samples of $\mathrm{Oct}(\vect{n})$,
$\vect{n}\sim\mathrm{Unif}(S^{2})$, and shares it between $\xi$ and $\eta$.

\subsection{MSE-optimal bit allocation}
\label{sec:method:bitsplit}

\ours quantizes each triplet in a total budget of
$B_{\mathrm{tri}}=2b_{\mathrm{dir}}+b_{\mathrm{nrm}}$ bits, where
$b_{\mathrm{dir}}$ bits go to each octahedral coordinate and
$b_{\mathrm{nrm}}$ bits go to the triplet norm. We parameterise the allocations
around an integer $b$, with a uniform reference
$b_{\mathrm{dir}}{=}b_{\mathrm{nrm}}{=}b$ giving $B_{\mathrm{tri}}{=}3b$.
This uniform split is sub-optimal in the \emph{squared-error} sense for
any reasonable $d$.

\inlinesection{MSE budget per triplet.}
Writing the encoder output as
$\hat{\vect{t}}_i = \hat{\rho}_i\,\vect{n}(\hat{\xi}_i,\hat{\eta}_i)$
and adding/subtracting $\rho_i\hat{\vect{n}}_i$ gives the bound
\begin{align}
    \norm{\vect{t}_i - \hat{\vect{t}}_i}_2^{2}
    &\leq 2(\rho_i - \hat{\rho}_i)^{2}
       + 2\rho_i^{2}\,\norm{\vect{n}_i - \hat{\vect{n}}_i}_2^{2},
    \label{eq:triplet-mse}
\end{align}
tight up to a $1{+}o(1)$ factor at any reasonable bit width.
Under the rotated-sphere prior $\rho_i$ and $\vect{n}_i$ are independent,
so expectations factor. By Panter-Dite high-rate
distortion~\citep{panter1951quantization,gersho1982structure}, a 1-D
Lloyd-Max quantizer with $b$ bits and source variance $\sigma^2$
incurs $D\approx C\sigma^2 4^{-b}$. The first term therefore contributes
$2C_{\rho}\,\sigma_{\rho}^{2}\,4^{-b_{\mathrm{nrm}}}$ with $\sigma_{\rho}^{2}$
to the variance of \eqn{rho-pdf}. The two scalar codebooks on $(\xi,\eta)$
pull the squared error back to $S^{2}$ through the constant-per-octant
Jacobian; absorbing that Jacobian and the factor of two into an effective
directional variance $\sigma_{\vect{n}}^{2}$, the second term contributes
$2\expect[\rho_i^{2}]\,C_n\,\sigma_{\vect{n}}^{2}\,4^{-b_{\mathrm{dir}}}
 = (6/d)\,C_n\,\sigma_{\vect{n}}^{2}\,4^{-b_{\mathrm{dir}}}$.
By \eqn{rho-pdf}, $\sigma_{\rho}^{2}=\mathcal{O}(d^{-1})$ while
$\sigma_{\vect{n}}^{2}=\mathcal{O}(1)$, so direction errors remain
order-one on $S^2$ even after the weighting of $\rho_i^2$.

\inlinesection{Lagrangian optimum.}
Minimizing
\begin{equation}
\label{eq:mse-budget}
    \expect\!\left[\norm{\vect{t}_i - \hat{\vect{t}}_i}_2^{2}\right]
    \propto
    2C_{\rho}\,\sigma_{\rho}^{2}\,4^{-b_{\mathrm{nrm}}}
    \;+\;(6/d)\,C_{n}\,\sigma_{\vect{n}}^{2}\,4^{-b_{\mathrm{dir}}}
\end{equation}
subject to $2b_{\mathrm{dir}}+b_{\mathrm{nrm}}=B_{\mathrm{tri}}$ gives
\begin{equation}
\label{eq:bitsplit-opt}
    b_{\mathrm{dir}}^{\star} - b_{\mathrm{nrm}}^{\star}
    = \log_{4}\!\left(
        \frac{3\,C_n\,\sigma_{\vect{n}}^{2}}
             {2d\,C_\rho\,\sigma_{\rho}^{2}}\right).
\end{equation}

Substituting the known $\sigma_\rho^2 = O(d^{-1})$ and $\sigma_{\vect{n}}^2 = O(1)$, the asymptotic bit gap is independent of key dimensionality $d$ and also notably independent of total bit budget $B_{\mathrm{tri}}$: 
$b_{\mathrm{dir}}^{\star} - b_{\mathrm{nrm}}^{\star} = \mathcal{O}(1)$.

\inlinesection{Empirical verification.}
On synthetic Gaussian keys at $d{=}128$ we sweep the diagonal
$(b{+}\delta,b{-}\delta)$, $\delta\in\{-2,\ldots,+2\}$, around each
uniform reference $b\in\{2,3,4\}$. The MSE landscape is sharply convex
in $\delta$ with minimum at $\delta{=}+1$, i.e.\ at $(b{+}1,b{-}1)$, for
every $b$ tested; relative to uniform $(b,b)$ the implemented split
\emph{reduces} MSE by $31$--$41\%$, while every other diagonal step
\emph{raises} it (by $+44$ to $+73\%$ at $\delta{=}+2$, and by an
order of magnitude or more at $\delta{=}-2$). The complete sweep is in
App.~\ref{app:bit-split-sweep}; \Section{results} shows that the same
$(b{+}1,b{-}1)$ split minimizes downstream error across every modality
we test.

\subsection{Codebooks}
\label{sec:method:codebooks}

Two Lloyd-Max codebooks suffice: $\mathcal{C}_{\rho}(d,b_{\mathrm{nrm}})$
on $[0,1]$ matched to \eqn{rho-pdf}, and $\mathcal{C}_{\xi}(b_{\mathrm{dir}})$
on $[-1,1]$ matched to the empirical $\xi$ marginal. Both are trained
off-line via the standard alternating
assignment/update Lloyd-Max iteration to distortion $10^{-10}$, are
serialized to disk, and are tiny ($\leq 32{+}8$ fp32 centroids per
$(d,b)$, $\approx 160$\,B). They depend only on $(d,b_{\mathrm{dir}},
b_{\mathrm{nrm}})$, without data-dependent calibration.

\subsection{Joint rounding of $(\xi_i,\eta_i,\rho_i)$}
\label{sec:method:rounding}

Given the bit split and the codebooks $\mathcal{C}_{\xi},
\mathcal{C}_{\rho}$ of \sect{method:codebooks}, the encoder still
chooses which code tuple $(\hat{\xi}_i,\hat{\eta}_i,\hat{\rho}_i)$ to
emit. Three independent nearest-centroid rounds under
\eqn{oct-marginal} and \eqn{rho-pdf} are \emph{marginal}-optimal but
not \emph{joint}-optimal: the decoder of \eqn{oct-decode} is nonlinear
in $(\xi,\eta)$ and multiplicative in $\rho$, so the
product-of-scalar-rounds does not in general minimize
\begin{equation}
\label{eq:triplet-joint-objective}
    \ell(\hat{\xi}_i,\hat{\eta}_i,\hat{\rho}_i)
    \coloneqq
    \norm{\vect{t}_i-\hat{\rho}_i\,\vect{n}(\hat{\xi}_i,\hat{\eta}_i)}_2^{2}.
\end{equation}
This is the octahedral analog of the ``optimal rounding'' pass for
tangent-frame codecs in graphics~\citep{kapoulkine2026tangent},
extended to include $\rho$ in the joint.

\inlinesection{Simplification.}
Expanding \eqn{triplet-joint-objective} with $\norm{\vect{n}(\cdot)}_2=1$ gives
\begin{equation}
\label{eq:triplet-joint-expand}
    \ell = \rho_i^{2}-2\hat{\rho}_i\,s_i(\hat{\xi}_i,\hat{\eta}_i)+\hat{\rho}_i^{2},
    \qquad
    s_i(\hat{\xi}_i,\hat{\eta}_i)
    \coloneqq \vect{t}_i^{\!\top}\vect{n}(\hat{\xi}_i,\hat{\eta}_i).
\end{equation}
For any fixed direction candidate, the optimal $\hat{\rho}_i$ is the
$\mathcal{C}_{\rho}$ centroid nearest to $s_i$ (not to $\norm{\vect{t}_i}_2$),
and the joint minimum reduces to maximizing $s_i$ on the direction
candidates: $(\hat{\xi}_i,\hat{\eta}_i) = \arg\max s_i$, then
$\hat{\rho}_i = \arg\min_{c\in\mathcal{C}_{\rho}}|c-\mathrm{clip}_{[0,1]}(s_i^{*})|$.
Direction selection therefore decouples from $\rho$ selection.

\inlinesection{Local 3$\times$3 candidate set.}
The full direction argmax runs over $2^{2b_{\mathrm{dir}}}$ candidates.
In practice, the Lloyd scalar seed $(i_\xi,i_\eta)$ is at most one index
away from the joint optimum at every bit width we measured, so \ours
enumerates only the nine candidates
$\{(i_\xi+\delta_x,\,i_\eta+\delta_y)\mid \delta_x,\delta_y\in\{-1,0,1\}\}$,
clamped to the codebook range. Across $10^{4}$ random rotated triplets
in $d{=}128$ and $b_{\mathrm{dir}}\in\{2,\ldots,5\}$, this search
was \emph{byte-identical} to the full grid search in all buckets at a
fraction of the cost (App.~\ref{app:rounding-ablation}).

\inlinesection{Format invariance.}
Only the \emph{encoder} changes; the bitstream layout, codebooks,
and decoder of \eqn{oct-decode} are untouched, so joint rounding
 does not require a decoder change. Every deployed \ours state (with or
without QJL) is decoded by the same fused attention kernel of
\sect{method:score}.
\Algo{encode} in App.~\ref{app:algos} writes out both variants; we run
\texttt{local\_3x3} as the default throughout \Section{results}.

\subsection{Score path and the optional 1-bit QJL residual}
\label{sec:method:score}

At decode time, the rotated-frame inner product factorizes over triplets:
\begin{equation}
\label{eq:score-exact}
    \vect{q}^{\!\top}\hat{\vect{k}}
    = \gamma\,\vect{q}_{\mathrm{rot}}^{\!\top}\hat{\vect{u}}
    = \gamma\!\sum_{i=0}^{n_{\mathrm{tri}}-1}
      \hat{\rho}_i\,\vect{q}_{\mathrm{rot},i}^{\!\top}\hat{\vect{n}}_i,
\end{equation}
where $\vect{q}_{\mathrm{rot}}=\matt{R}\vect{q}$ and
$\hat{\vect{n}}_i=\mathrm{Oct}^{-1}(\mathcal{C}_{\xi}[I^{\mathrm{dir}}_{i,0}],
\mathcal{C}_{\xi}[I^{\mathrm{dir}}_{i,1}])$, $\hat{\rho}_i=\mathcal{C}_{\rho}[I^{\mathrm{nrm}}_i]$.
Only $2n_{\mathrm{tri}}$ direction- and $n_{\mathrm{tri}}$
norm-centroid loads are required; $\hat{\matt{K}}$ never materialized. The encoder and a fused split-K flash decoder are in
App.~\ref{app:algos}.

\inlinesection{1-bit QJL residual (\ours-QJL).}
MSE-optimal scalar quantizers are biased in the dot
product~\citep{zandieh2025turboquant}. We optionally attach a
QJL~\citep{zandieh2024qjl} sketch of the rotated-frame residual
$\vect{r}\coloneqq\vect{u}-\hat{\vect{u}}$. With
$\matt{R}'=\matt{H}\,\mathrm{diag}(\vect{s}')$ a second rotation with
independent seed, we store $\vect{\sigma}=
\mathrm{sign}(\matt{R}'\vect{r})\in\{\pm1\}^{d}$ and a residual norm
$\gamma_{r}=\norm{\vect{r}}_2$ (fp16). The QJL estimator
$\widehat{\vect{q}_{\mathrm{rot}}^{\!\top}\vect{r}}=
\sqrt{\pi/(2d)}\,\gamma_{r}\,(\matt{R}'\vect{q}_{\mathrm{rot}})^{\!\top}\vect{\sigma}$
is unbiased under the ideal QJL model, with variance $\mathcal{O}(1/d)$; the
implementation uses the same scaling with a structured WHT rotation and an
fp16-rounded $\gamma_r$. The corrected score is
$\vect{q}^{\!\top}\hat{\vect{k}}+\gamma\,
\widehat{\vect{q}_{\mathrm{rot}}^{\!\top}\vect{r}}$.

\section{Experiments}
\label{sec:results}

We compare \ours and \ours-QJL against three rotation-preconditioned
codecs sharing the same Walsh-Hadamard rotation, $V$ codec, and
residual window:
\textbf{TurboQuant-MSE}~\citep{zandieh2025turboquant} (per-coordinate
Lloyd-Max),
\textbf{TurboQuant-QJL}~\citep{zandieh2024qjl,zandieh2025turboquant}
(MSE stage $+$ 1-bit JL residual), and
\textbf{PolarQuant}~\citep{han2025polarquant} (recursive polar). The
only variable across rows is the $K$ codec, and every comparison is
matched at the same symmetric $K{=}V$ bit width.

\inlinesection{Modalities.} (i)~A \emph{synthetic probe} of isotropic
Gaussian keys at $d{=}128$, the regime in which the rotation-Beta-Lloyd
baseline is provably near-optimal. (ii)~Long-context language modelling
with \texttt{Qwen2.5-7B-Instruct-1M}~\citep{dubey2024llama3}: 7B
parameters, GQA~\citep{ainslie2023gqa}, 28 layers, $d_h{=}128$, 1M
native context. (iii) Two Wan-1.3B autoregressive video DiTs at
$d_h{=}64$ and 30 blocks: chunk-wise CausVid~\citep{yin2025causvid}
(3-frame chunks) and frame-wise Causal Forcing~\citep{zhu2026causal}.
(iv)~The 16-block next-scale autoregressive audio model
AAR~\citep{qiu2024efficient}. Default recipe: short residual window of
native-precision tokens/frames/scales and value-side group size
$\in\{16,32\}$. The video and audio cross-codec rows use the
unprotected default; the LLM cross-codec rows use the boundary-1 $K$
recipe described in \sect{results:llm} as a setup prerequisite.
Compression ratios are
$(\text{fp16 KV bytes})/(\text{compressed KV bytes})$.

\subsection{Synthetic rate-quality and needle retrieval}
\label{sec:results:synthetic}

We draw $n{=}1024$ Gaussian keys and $16$ Gaussian queries at $d{=}128$
and average over $64$ seeds, reporting reconstruction cosine, per-coord
MSE, and inner-product (IP) absolute error
$|\vect{q}^{\!\top}\vect{k}-\mathrm{score}(\vect{q},\vect{k})|$ with each codec's
paper-claimed estimator (cf.\ TurboQuant Fig.~1--2~\citep{zandieh2025turboquant}).
For needle-in-a-haystack we plant one key in $T{=}2048$ Gaussian
distractors with $10\%$ noisy query and report softmax mass on the
needle, averaged over $128$ seeds (the fp32 baseline concentrates
$0.960$).

\begin{figure}[t]
    \centering
    \begin{subfigure}[b]{0.46\linewidth}
        \includegraphics[width=\linewidth]{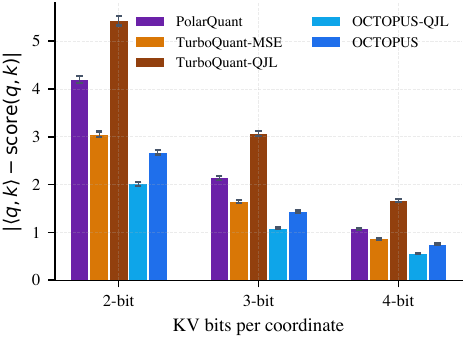}
        \caption{IP absolute error ($d{=}128$)}
    \end{subfigure}\hfill
    \begin{subfigure}[b]{0.46\linewidth}
        \includegraphics[width=\linewidth]{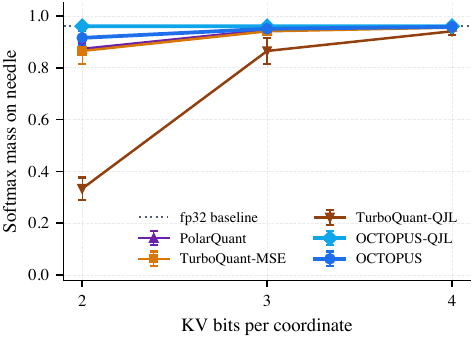}
        \caption{Synthetic needle retrieval}
    \end{subfigure}
    \titlecaption{Synthetic fidelity}{%
      (a)~\ours-QJL is best at every bit width; \ours alone beats
      every non-QJL baseline. (b)~\ours-QJL tracks fp32 to within
      $0.001$; TurboQuant-QJL drops to near-uniform at $b{=}2$.}%
    \label{fig:synth_ip}
\end{figure}

\Table{synthetic} and \fig{synth_ip}: \ours has the best
reconstruction fidelity of any rotation codec at every bit width,
with MSE $1.3\times$ below the per-coordinate-optimal TurboQuant-MSE
at $b{=}4$ and $2.4\times$ below PolarQuant at $b{=}2$.
\ours-QJL drives the IP error $3\times$ below TurboQuant-QJL at a matched
rate (the latter spends one bit on its stage-1 quantizer, leaving
the reconstruction one bit worse). On the synthetic needle, \ours-QJL
tracks the fp32 baseline to within $0.001$; at $b{=}2$, \ours
preserves $0.92$ of the softmax mass vs.\ $0.86$/$0.87$/$0.33$.

\begin{figure}[t]
\centering
\titlecaptionof{table}{Synthetic reconstruction fidelity at $d=128$}{Isotropic Gaussian keys/queries, averaged over $64$ seeds. Reconstruction MSE per coord. Best per bit-width block is \textbf{bold}, runner-up \underline{underlined}.}
\label{tab:synthetic}
\footnotesize
\setlength{\tabcolsep}{5pt}
\begin{tabular}{llrrr}
bits & codec & cos ($\uparrow$) & MSE ($\downarrow$) & IP abs err ($\downarrow$) \\
\midrule
2 & TurboQuant-MSE & \underline{0.9406} & \underline{0.1161} & 3.054 \\
2 & TurboQuant-QJL & 0.7994 & 0.3610 & 5.427 \\
2 & PolarQuant & 0.8902 & 0.2197 & 4.200 \\
2 & \textbf{\ours} & \textbf{0.9547} & \textbf{0.0897} & \underline{2.682} \\
2 & \textbf{\ours-QJL} & \textbf{0.9547} & \textbf{0.0897} & \textbf{2.015} \\
\midrule
3 & TurboQuant-MSE & \underline{0.9831} & \underline{0.0340} & 1.650 \\
3 & TurboQuant-QJL & 0.9406 & 0.1161 & 3.072 \\
3 & PolarQuant & 0.9715 & 0.0571 & 2.142 \\
3 & \textbf{\ours} & \textbf{0.9871} & \textbf{0.0260} & \underline{1.444} \\
3 & \textbf{\ours-QJL} & \textbf{0.9871} & \textbf{0.0260} & \textbf{1.084} \\
\midrule
4 & TurboQuant-MSE & \underline{0.9954} & \underline{0.0094} & 0.866 \\
4 & TurboQuant-QJL & 0.9831 & 0.0340 & 1.660 \\
4 & PolarQuant & 0.9928 & 0.0145 & 1.079 \\
4 & \textbf{\ours} & \textbf{0.9965} & \textbf{0.0071} & \underline{0.753} \\
4 & \textbf{\ours-QJL} & \textbf{0.9965} & \textbf{0.0071} & \textbf{0.565} \\
\end{tabular}

\end{figure}

\subsection{Long-context language modelling (Qwen2.5-7B-Instruct-1M)}
\label{sec:results:llm}

\begin{figure}[t]
\centering
\titlecaptionof{table}{Long-context LM on \texttt{Qwen2.5-7B-Instruct-1M}}{WikiText-2 / C4 perplexity at context $4096$, symmetric $K{=}V$. Every row uses the same recipe (residual window $32$, $V$ group size $32$, K-side protection on the outer transformer block at each end---a stability prerequisite for this model, not a contribution). Deltas are vs.\ fp16. KV$\times$ $=$ fp16 bytes\,/\,compressed bytes. Best per bit-width block is \textbf{bold}, runner-up \underline{underlined}.}
\label{tab:llm}
\scriptsize
\setlength{\tabcolsep}{4pt}
\begin{tabular}{llrrrrr}
bits & codec & WikiText-2 & $\Delta$\% & C4 & $\Delta$\% & KV$\times$ \\
\midrule
-- & fp16 baseline & 10.033 & 0.00 & 12.701 & 0.00 & $1.0\times$ \\
\midrule
4 & TurboQuant-MSE & \underline{10.340} & \underline{+3.1} & 12.921 & \underline{+1.7} & $2.2\times$ \\
4 & TurboQuant-QJL & 10.836 & +8.0 & 13.699 & +7.9 & $2.2\times$ \\
4 & PolarQuant & 10.473 & +4.4 & 13.091 & +3.1 & $2.2\times$ \\
4 & \textbf{\ours} & \textbf{10.306} & \textbf{+2.7} & \underline{12.896} & \textbf{+1.5} & $2.2\times$ \\
4 & \textbf{\ours-QJL} & \textbf{10.306} & \textbf{+2.7} & \textbf{12.893} & \textbf{+1.5} & $2.0\times$ \\
\midrule
3 & TurboQuant-MSE & 10.899 & \underline{+8.6} & 13.761 & +8.3 & $2.6\times$ \\
3 & TurboQuant-QJL & 15.093 & +50.4 & 20.308 & +59.9 & $2.5\times$ \\
3 & PolarQuant & 11.612 & +15.7 & 14.716 & +15.9 & $2.6\times$ \\
3 & \textbf{\ours} & \textbf{10.753} & \textbf{+7.2} & \textbf{13.446} & \textbf{+5.9} & $2.5\times$ \\
3 & \textbf{\ours-QJL} & \underline{10.754} & \textbf{+7.2} & \underline{13.474} & \underline{+6.1} & $2.3\times$ \\
\midrule
2 & TurboQuant-MSE & 16.354 & \underline{+63.0} & 22.536 & +77.4 & $3.0\times$ \\
2 & TurboQuant-QJL & 87.490 & +772.0 & 184.034 & +1349.0 & $3.0\times$ \\
2 & PolarQuant & 28.759 & +186.6 & 61.486 & +384.1 & $3.0\times$ \\
2 & \textbf{\ours} & \underline{13.517} & \textbf{+34.7} & \underline{17.976} & \underline{+41.5} & $2.9\times$ \\
2 & \textbf{\ours-QJL} & \textbf{13.511} & \textbf{+34.7} & \textbf{17.955} & \textbf{+41.4} & $2.6\times$ \\
\end{tabular}

\end{figure}

Following~\citet[\S5]{zandieh2025turboquant} we report WikiText-2 and
C4 perplexity (PPL) ($512$-token blocks, $8$ chunks) and a multi-key
needle-in-a-haystack sweep~\citep{kamradt2023needle,hsieh2024ruler}
($4$k--$128$k context; $4{+}1$ needles with random $8$-char magic
values, exact-match scoring). Recipe: residual window $32$, $V$ group
size $32$, $K$ held at fp16 on both boundary blocks
(``boundary-1'')---a stability prerequisite, not a contribution: every
rotated codec diverges to PPL~$>$$10^3$ without it. All \Table{llm}
rows share this setup.

\inlinesection{Quality.} \ours leads every rotation codec at every
bit width (\Table{llm}). In $b{=}4$ the WikiText-2 gap is modest
($+2.7\%$ vs.\ $+3.1/4.4/8.0\%$); in $b{=}2$ the separation is
decisive ($+34.7\%$ vs.\ $+63/187/772\%$).

\inlinesection{Needle-in-a-haystack.}
Multi-key random-value retrieval ($4$k--$128$k context, $4$ samples
per cell; App.~\ref{app:llm-niah}). At $b{=}4$ all codecs reach
$1.00$. At $b{=}3$, \ours holds $1.00$; PolarQuant drops to $0.86$
average. At $b{=}2$ (\fig{llm_pareto}), only \ours/$0.81$ and
\ours-QJL/$0.83$ retain recall; PolarQuant and TurboQuant-QJL collapse
($0.04/0.01$), tracking their perplexity divergence.

\begin{figure}[t]
    \centering
    \begin{subfigure}[b]{0.4\linewidth}
        \includegraphics[width=\linewidth]{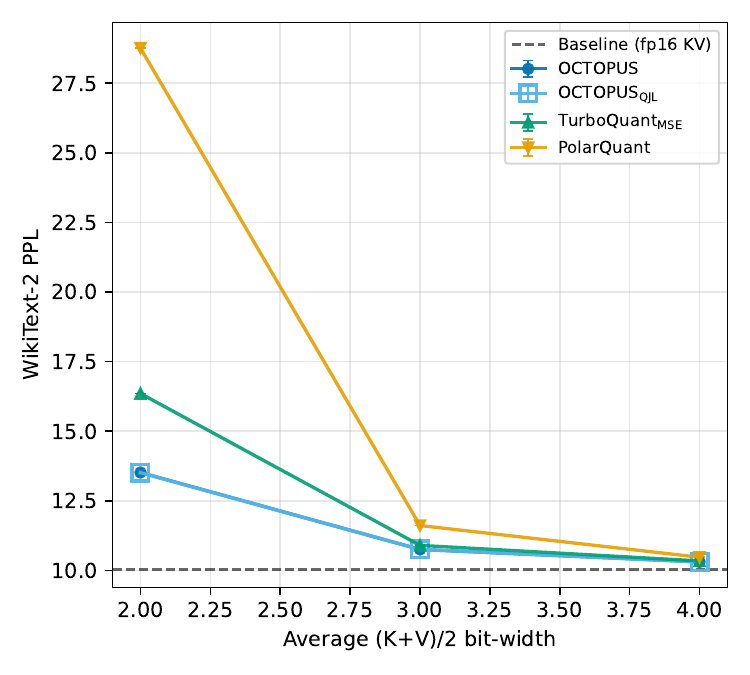}
        \caption{PPL ($\downarrow$) vs.\ bits}
    \end{subfigure}\hfill
    \begin{subfigure}[b]{0.4\linewidth}
        \includegraphics[width=\linewidth]{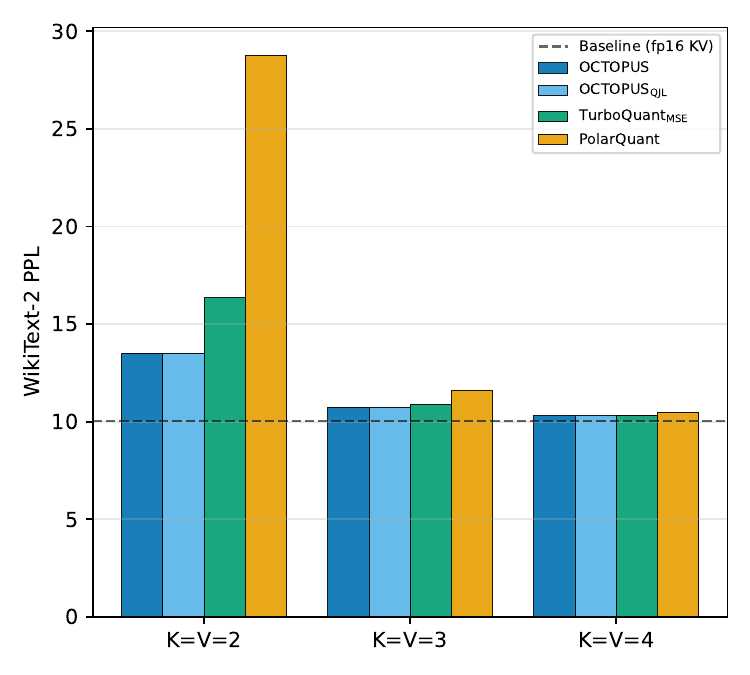}
        \caption{$\Delta$PPL ($\downarrow$) per bit width}
    \end{subfigure}\\[4pt]
    \begin{subfigure}[b]{0.4\linewidth}
        \includegraphics[width=\linewidth]{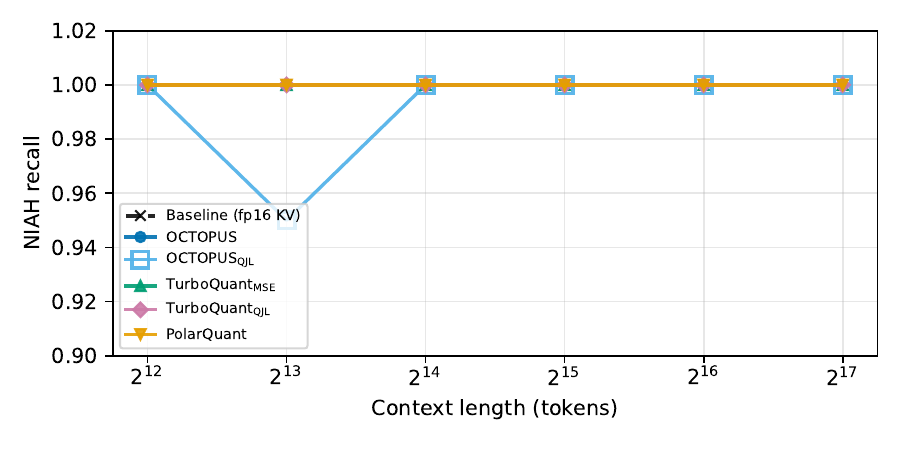}
        \caption{NIAH recall ($\uparrow$) ($b{=}4$)}
    \end{subfigure}\hfill
    \begin{subfigure}[b]{0.4\linewidth}
        \includegraphics[width=\linewidth]{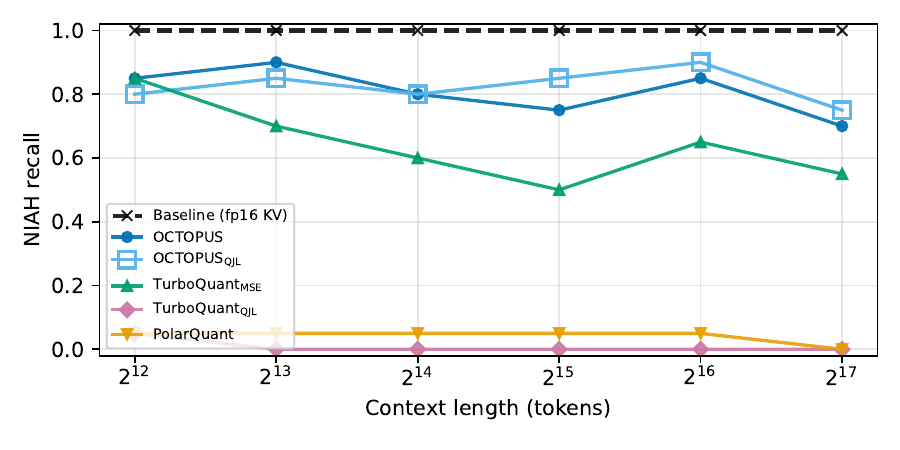}
        \caption{NIAH recall ($\uparrow$) ($b{=}2$)}
    \end{subfigure}
    \titlecaption{Qwen2.5-7B rate-quality and needle recall}{%
      \ours does not collapse at $b{=}2$ on either PPL or retrieval;
      at $b{=}4$ all codecs retain baseline recall.}%
    \label{fig:llm_pareto}
\end{figure}

\subsection{Autoregressive video and audio}
\label{sec:results:av}

\inlinesection{Setup.} The video experiments use two Wan-1.3B
autoregressive DiTs with 30 blocks, $d_h{=}64$, and bf16 activations:
CausVid~\citep{yin2025causvid}, which generates in $3$-frame chunks,
and Causal Forcing~\citep{zhu2026causal}, which advances frame by
frame. We compress the attention KV cache during generation with a
residual window of one native-precision frame, value group size
$g{=}32$, and no boundary-block protection. For each model and bit
width, every codec is run on the same $100$ prompts with byte-identical
initial noise; the reported deltas therefore isolate the codec rather
than prompt or sampling variation. We measure LPIPS~\citep{zhang2018lpips},
PSNR, SSIM, and latent cosine against the uncompressed rollout.

The audio experiment uses AAR~\citep{qiu2024efficient}, a 16-block
next-scale autoregressive model. We follow the released CLAP-conditioned
inference path: $100$ random $10$\,s AudioSet-20k clips are encoded by
CLAP and used as conditioning, while the model generates the
corresponding audio continuation/sample under compressed KV. The cache
recipe matches the video sweep except for the autoregressive unit and
group size: residual window one native-precision scale, $V$ group
$g{=}16$, and no per-layer protection. We report LSD, log-mel MSE, SNR,
and latent cosine against the uncompressed AAR output.

\begin{table}[t]
\centering
\titlecaption{Compressed-KV video and audio, symmetric $K{=}V$}{Per-prompt min / $\mu$ / max across all prompts. CausVid: residual window $1$ frame; Causal Forcing: $1$ frame, frame-wise; AAR: $100$ random $10$\,s AudioSet-20k clips as CLAP-audio conditioning, residual window $1$ scale, $g{=}16$. KV$\times$ is the video ($g{=}32$) compression ratio. Mean ($\mu$) is \textbf{bold} for best, \underline{underlined} for runner-up per bit-width block.}
\label{tab:video_audio}
\tiny
\setlength{\tabcolsep}{2.5pt}
\resizebox{\columnwidth}{!}{
\begin{tabular}{llr|rrr|rrr|rrr|rrr|rrr|rrr}
\multicolumn{3}{c|}{} & \multicolumn{6}{c|}{CausVid} & \multicolumn{6}{c|}{Causal Forcing} & \multicolumn{6}{c}{AAR (audio)} \\
\multicolumn{3}{c|}{} & \multicolumn{3}{c}{LPIPS$\downarrow$} & \multicolumn{3}{c|}{PSNR$\uparrow$} & \multicolumn{3}{c}{LPIPS$\downarrow$} & \multicolumn{3}{c|}{PSNR$\uparrow$} & \multicolumn{3}{c}{LSD$\downarrow$} & \multicolumn{3}{c}{SNR$\uparrow$} \\
b & codec & KV$\times$ & min & $\mu$ & max & min & $\mu$ & max & min & $\mu$ & max & min & $\mu$ & max & min & $\mu$ & max & min & $\mu$ & max \\
\midrule
4 & TurboQuant-MSE & $2.4\times$ & 0.008 & 0.045 & 0.130 & 17.4 & 26.5 & 39.5 & 0.140 & 0.334 & 0.637 & 9.2 & 14.6 & 21.8 & 0.0 & 6.4 & 9.9 & -2.1 & \underline{2.1} & 120.0 \\
4 & TurboQuant-QJL & $2.4\times$ & 0.014 & 0.096 & 0.265 & 15.5 & 22.6 & 34.8 & 0.213 & 0.421 & 0.663 & 8.1 & 13.1 & 18.6 & 0.0 & \textbf{6.2} & 10.0 & -2.4 & 1.8 & 120.0 \\
4 & PolarQuant & $2.4\times$ & 0.006 & \textbf{0.037} & 0.124 & 19.5 & \textbf{27.8} & 42.2 & 0.113 & \textbf{0.301} & 0.582 & 11.1 & \textbf{15.4} & 24.2 & 0.0 & \underline{6.3} & 9.6 & -2.2 & \textbf{2.2} & 120.0 \\
4 & \textbf{\ours} & $2.3\times$ & 0.005 & \underline{0.038} & 0.115 & 19.3 & \underline{27.5} & 42.8 & 0.142 & \underline{0.309} & 0.522 & 10.2 & \underline{15.2} & 20.2 & 0.0 & \textbf{6.2} & 9.5 & -2.8 & \textbf{2.2} & 120.0 \\
4 & \textbf{\ours-QJL} & $2.2\times$ & 0.006 & \underline{0.038} & 0.117 & 19.3 & \underline{27.5} & 42.6 & 0.153 & 0.310 & 0.552 & 9.6 & \underline{15.2} & 20.3 & 0.0 & \textbf{6.2} & 9.8 & -4.0 & \textbf{2.2} & 120.0 \\
\midrule
3 & TurboQuant-MSE & $2.7\times$ & 0.013 & 0.098 & 0.263 & 15.5 & 22.6 & 35.5 & 0.207 & 0.423 & 0.661 & 8.5 & \underline{13.1} & 18.8 & 0.0 & 6.5 & 10.0 & -2.3 & \underline{1.6} & 120.0 \\
3 & TurboQuant-QJL & $2.7\times$ & 0.076 & 0.262 & 0.449 & 12.8 & 17.8 & 28.0 & 0.624 & 0.779 & 0.910 & 5.9 & 8.4 & 11.2 & 9.6 & 12.7 & 16.2 & -19.5 & -5.4 & -0.1 \\
3 & PolarQuant & $2.8\times$ & 0.013 & \underline{0.093} & 0.263 & 16.1 & \underline{22.8} & 34.9 & 0.218 & 0.402 & 0.691 & 9.0 & \underline{13.1} & 19.5 & 0.0 & \textbf{6.3} & 10.1 & -2.9 & \textbf{1.7} & 120.0 \\
3 & \textbf{\ours} & $2.7\times$ & 0.012 & \textbf{0.078} & 0.228 & 16.5 & \textbf{23.7} & 36.0 & 0.196 & \underline{0.390} & 0.628 & 8.4 & \textbf{13.5} & 20.8 & 0.0 & \underline{6.4} & 10.1 & -2.6 & 1.5 & 120.0 \\
3 & \textbf{\ours-QJL} & $2.5\times$ & 0.012 & \textbf{0.078} & 0.225 & 16.5 & \textbf{23.7} & 36.4 & 0.207 & \textbf{0.389} & 0.624 & 8.8 & \textbf{13.5} & 20.4 & 0.0 & \underline{6.4} & 10.1 & -3.0 & \underline{1.6} & 120.0 \\
\midrule
2 & TurboQuant-MSE & $3.2\times$ & 0.055 & 0.261 & 0.450 & 12.9 & \underline{17.9} & 28.1 & 0.616 & 0.777 & 0.907 & 5.9 & 8.4 & 11.6 & 10.2 & 12.7 & 15.7 & -17.6 & -5.3 & 0.3 \\
2 & TurboQuant-QJL & $3.2\times$ & 0.265 & 0.579 & 0.830 & 9.5 & 13.1 & 19.8 & 0.708 & 0.816 & 0.997 & 5.5 & 7.1 & 8.2 & 10.5 & 13.2 & 16.5 & -28.9 & -6.0 & -0.2 \\
2 & PolarQuant & $3.3\times$ & 0.045 & \underline{0.251} & 0.469 & 13.0 & \underline{17.9} & 28.3 & 0.575 & \underline{0.727} & 0.898 & 5.5 & \underline{8.6} & 11.2 & 9.9 & 12.6 & 16.0 & -21.3 & -5.3 & 2.3 \\
2 & \textbf{\ours} & $3.1\times$ & 0.029 & \textbf{0.178} & 0.366 & 13.8 & \textbf{19.7} & 30.9 & 0.341 & \textbf{0.581} & 0.821 & 6.8 & \textbf{10.9} & 15.0 & 0.0 & \textbf{6.8} & 13.6 & -9.6 & \textbf{1.1} & 120.0 \\
2 & \textbf{\ours-QJL} & $2.8\times$ & 0.028 & \textbf{0.178} & 0.364 & 13.9 & \textbf{19.7} & 30.9 & 0.377 & \textbf{0.581} & 0.815 & 6.7 & \textbf{10.9} & 15.0 & 0.0 & \underline{6.9} & 14.4 & -8.2 & \underline{1.0} & 120.0 \\
\end{tabular}
}

\end{table}

\inlinesection{Findings.} \Table{video_audio} reports per-prompt
min/$\mu$/max. At $b{=}4$ all codecs overlap ($\leq\!3\%$).
At $b{=}2$ the picture changes sharply: on Causal Forcing,
TurboQuant-QJL reaches a worst-case LPIPS of $1.00$ and a mean of
$0.82$---effectively random noise---while \ours stays at
$0.58$/$0.82$ (min/max).
On audio, the $10$\,s AudioSet-conditioned sweep is forgiving at
$b{=}4$ (all codecs lie within $0.19$\,dB LSD), but separates sharply at
$b{=}2$: TurboQuant-MSE, TurboQuant-QJL, and PolarQuant rise to
$12.6$--$13.2$\,dB mean LSD with negative mean SNR, while \ours remains
at $6.75$\,dB LSD and $+1.07$\,dB SNR. Even PolarQuant, the strongest
non-\ours baseline, degrades $1.4{\times}$ faster than \ours on mean
LPIPS as bits decrease (CausVid: $0.25$ vs.\ $0.18$). Stills from the videos are provided in
App.~\ref{app:stills}.

\subsection{Cross-modality patterns}
\label{sec:results:patterns}

All four modalities show the same pattern.
\textbf{(i)~\ours matches or beats every rotation baseline in the
low-bit regimes where compression quality matters most.} Exceptions:
$b{=}4$ video (Polar within $3\%$) and AAR at $b{=}3$, where PolarQuant
is slightly better on mean LSD/SNR under $10$\,s AudioSet conditioning.
\textbf{(ii)~Competing codecs degrade catastrophically below $b{=}4$:}
TurboQuant-QJL collapses to perceptual noise on CF at $b{=}2$
(max LPIPS\,${\approx}\,1.0$); PolarQuant's worst prompt at $b{=}2$
hits LPIPS\,$0.90$.  \ours's worst prompt stays at $0.82$---still
degraded, but coherent.  The extra $b_{\mathrm{dir}}$ bit
(\eqn{bitsplit-opt}) provides a disproportionate MSE reduction at
tight budgets.
\textbf{(iii)~QJL buys IP accuracy, not reconstruction-path quality;}
\ours-QJL fits only score-attention deployments
(\Table{qjl_accounting}).

\inlinesection{Limitations.} 
The improved rate-quality point is not free in wall-clock time:
\ours adds more arithmetic than scalar Lloyd-Max decoding and remains slower than a bf16 SDPA path, so it is most attractive when KV bandwidth or capacity is the bottleneck (App.~\ref{app:kernel-speed}).

\section{Conclusion}
\label{sec:conclusion}

\ours is a rotation-preconditioned KV codec that quantizes the rotated
unit direction in contiguous triplets: an octahedral
map~\citep{engelhardt2008octahedron,cigolle2014survey} collapses each
3-coordinate block to a pair of scalars on $[-1,1]^{2}$, and
Lloyd-Max~\citep{lloyd1982least,max1960quantizing} quantizers matched
to the norm and oct-coordinate marginals reduce the triplet to three
integers under the asymmetric $(b{+}1,b{-}1)$ bit split. The codec
inherits the data-oblivious online guaranties of
TurboQuant~\citep{zandieh2025turboquant} and combines without
modification with the 1-bit QJL~\citep{zandieh2024qjl} residual.
Across text, video, and audio, it matches or beats prior
rotation codecs, with a lead that grows as bits drop. 
At $b{=}2$, \ours is often the only codec that
does not collapse in long-context recall, and the only codec that
retains usable perceptual quality in autoregressive video.

\bibliographystyle{ieeenat_fullname}
\bibliography{references}

@String{ARXIV            = {arXiv preprint}}

@String{CVPR             = {Conference on Computer Vision and Pattern Recognition (CVPR)}}

@String{EMNLP            = {Conference on Empirical Methods in Natural Language Processing (EMNLP)}}

@String{ICML             = {International Conference on Machine Learning (ICML)}}

@String{ITIT             = {IEEE Transactions on Information Theory}}

@String{MLS              = {Proceedings of Machine Learning and Systems}}

@String{NIPS             = {Neural Information Processing Systems (NeurIPS)}}

@String{TACL             = {Transactions of the Association for Computational Linguistics (TACL)}}

@article{cigolle2014survey,
  author           = {Cigolle, Zina H. and Donow, Sam and Evangelakos, Daniel and Mara, Michael and McGuire, Morgan and Meyer, Quirin},
  title            = {A Survey of Efficient Representations for Independent Unit Vectors},
  journal          = {Journal of Computer Graphics Techniques (JCGT)},
  year             = {2014},
  volume           = {3},
  number           = {2},
  pages            = {1--30},
  url              = {http://jcgt.org/published/0003/02/01/}
}

@inproceedings{engelhardt2008octahedron,
  author           = {Engelhardt, Thomas and Dachsbacher, Carsten},
  title            = {Octahedron Environment Maps},
  booktitle        = {International Symposium on Vision, Modeling, and Visualization (VMV)},
  year             = {2008}
}

@misc{kapoulkine2026tangent,
  author           = {Kapoulkine, Arseny},
  title            = {Quantizing tangent frames},
  year             = {2026},
  howpublished     = {Blog post, \url{https://zeux.io/2026/04/30/quantizing-tangent-frames/}},
  note             = {Accessed 2026-04-30}
}

@article{zandieh2024qjl,
  author           = {Zandieh, Amir and Daliri, Majid and Han, Insu},
  title            = {{QJL}: 1-Bit Quantized {JL} Transform for {KV} Cache Quantization with Zero Overhead},
  journal          = ARXIV,
  year             = {2024},
  url              = {https://arxiv.org/abs/2406.03482},
  eprint           = {2406.03482},
  archiveprefix    = {arXiv},
  primaryclass     = {cs.LG}
}

@article{zandieh2025turboquant,
  author           = {Zandieh, Amir and Daliri, Majid and Hadian, Majid and Mirrokni, Vahab},
  title            = {{TurboQuant}: Online Vector Quantization with Near-optimal Distortion Rate},
  journal          = ARXIV,
  year             = {2025},
  url              = {https://arxiv.org/abs/2504.19874},
  eprint           = {2504.19874},
  archiveprefix    = {arXiv},
  primaryclass     = {cs.LG}
}

@article{wu2025polarquant,
  author           = {Wu, Songhao and Lv, Ang and Feng, Xiao and Zhang, Yufei and Zhang, Xun and Yin, Guojun and Lin, Wei and Yan, Rui},
  title            = {{PolarQuant}: Leveraging Polar Transformation for Efficient Key Cache Quantization and Decoding Acceleration},
  journal          = ARXIV,
  year             = {2025},
  url              = {https://arxiv.org/abs/2502.00527},
  eprint           = {2502.00527},
  archiveprefix    = {arXiv},
  primaryclass     = {cs.CL}
}

@article{han2025polarquant,
  author           = {Han, Insu and Kacham, Praneeth and Karbasi, Amin and Mirrokni, Vahab and Zandieh, Amir},
  title            = {{PolarQuant}: Quantizing {KV} Caches with Polar Transformation},
  journal          = ARXIV,
  year             = {2025},
  url              = {https://arxiv.org/abs/2502.02617},
  eprint           = {2502.02617},
  archiveprefix    = {arXiv},
  primaryclass     = {cs.LG},
  note             = {Not to be confused with Wu et al.\ (arXiv:2502.00527), which shares the name ``PolarQuant'' but proposes a different method.}
}

@article{qiu2024efficient,
  author           = {Qiu, Kai and Li, Xiang and Chen, Hao and Sun, Jie and Wang, Jinglu and Lin, Zhe and Savvides, Marios and Raj, Bhiksha},
  title            = {Efficient Autoregressive Audio Modeling via Next-Scale Prediction},
  journal          = ARXIV,
  year             = {2024},
  url              = {https://arxiv.org/abs/2408.09027},
  eprint           = {2408.09027},
  archiveprefix    = {arXiv},
  primaryclass     = {cs.SD}
}

@inproceedings{yin2025causvid,
  author           = {Yin, Tianwei and Zhang, Qiang and Zhang, Richard and Freeman, William T. and Durand, Fr{\'e}do and Shechtman, Eli and Huang, Xun},
  title            = {From Slow Bidirectional to Fast Autoregressive Video Diffusion Models},
  booktitle        = CVPR,
  year             = {2025},
  url              = {https://arxiv.org/abs/2412.07772},
  eprint           = {2412.07772},
  archiveprefix    = {arXiv}
}

@article{zhu2026causal,
  author           = {Zhu, Hongzhou and Zhao, Min and He, Guande and Su, Hang and Li, Chongxuan and Zhu, Jun},
  title            = {Causal Forcing: Autoregressive Diffusion Distillation Done Right for High-Quality Real-Time Interactive Video Generation},
  journal          = ARXIV,
  year             = {2026},
  url              = {https://arxiv.org/abs/2602.02214},
  eprint           = {2602.02214},
  archiveprefix    = {arXiv}
}

@inproceedings{ainslie2023gqa,
  author           = {Ainslie, Joshua and Lee-Thorp, James and de Jong, Michiel and Zemlyanskiy, Yury and Lebron, Federico and Sanghai, Sumit},
  title            = {{GQA}: Training Generalized Multi-Query Transformer Models from Multi-Head Checkpoints},
  booktitle        = EMNLP,
  year             = {2023},
  pages            = {4895--4901}
}

@article{ashkboos2024quarot,
  author           = {Ashkboos, Saleh and Mohtashami, Amirkeivan and Croci, Maximilian L. and Li, Bo and Cameron, Pashmina and Jaggi, Martin and Alistarh, Dan and Hoefler, Torsten and Hensman, James},
  title            = {{QuaRot}: Outlier-Free 4-Bit Inference in Rotated {LLMs}},
  journal          = ARXIV,
  year             = {2024}
}

@article{cai2024pyramidkv,
  author           = {Cai, Zefan and Zhang, Yichi and Gao, Bofei and Liu, Yuliang and Liu, Tianyu and Lu, Keming and Xiong, Wayne and Dong, Yue and Chang, Baobao and Hu, Junjie and others},
  title            = {{PyramidKV}: Dynamic {KV} Cache Compression Based on Pyramidal Information Funneling},
  journal          = ARXIV,
  year             = {2024}
}

@article{chee2023quip,
  author           = {Chee, Jerry and Cai, Yaohui and Kuleshov, Volodymyr and De Sa, Christopher M.},
  title            = {{QuIP}: 2-Bit Quantization of Large Language Models with Guarantees},
  journal          = NIPS,
  year             = {2023},
  volume           = {36},
  pages            = {4396--4429}
}

@article{dettmers2022gpt3int8,
  author           = {Dettmers, Tim and Lewis, Mike and Belkada, Younes and Zettlemoyer, Luke},
  title            = {{GPT3.int8()}: 8-Bit Matrix Multiplication for Transformers at Scale},
  journal          = NIPS,
  year             = {2022},
  volume           = {35},
  pages            = {30318--30332}
}

@article{dong2024qaq,
  author           = {Shichen Dong and Wenfang Cheng and Jiayu Qin and Wei Wang},
  title            = {{QAQ}: Quality Adaptive Quantization for {LLM} {KV} Cache},
  journal          = {2025 IEEE/CVF International Conference on Computer Vision Workshops (ICCVW)},
  year             = {2024},
  doi              = {10.1109/ICCVW69036.2025.00267}
}

@article{dubey2024llama3,
  author           = {Dubey, Abhimanyu and Jauhri, Abhinav and Pandey, Abhinav and Kadian, Abhishek and Al-Dahle, Ahmad and Letman, Aiesha and Mathur, Akhil and Schelten, Alan and Yang, Amy and Fan, Angela and others},
  title            = {The {Llama 3} Herd of Models},
  journal          = ARXIV,
  year             = {2024}
}

@article{frantar2022gptq,
  author           = {Frantar, Elias and Ashkboos, Saleh and Hoefler, Torsten and Alistarh, Dan},
  title            = {{GPTQ}: Accurate Post-Training Quantization for Generative Pre-Trained Transformers},
  journal          = ARXIV,
  year             = {2022}
}

@article{fu2024data,
  author           = {Fu, Yao and Panda, Rameswar and Niu, Xinyao and Yue, Xiang and Hajishirzi, Hannaneh and Kim, Yoon and Peng, Hao},
  title            = {Data Engineering for Scaling Language Models to {128K} Context},
  journal          = ARXIV,
  year             = {2024}
}

@article{gersho1979asymptotically,
  author           = {Gersho, Allen},
  title            = {Asymptotically Optimal Block Quantization},
  journal          = {IEEE Transactions on Information Theory},
  year             = {1979},
  volume           = {25},
  number           = {4},
  pages            = {373--380}
}

@article{gersho1982structure,
  author           = {Gersho, Allen},
  title            = {On the Structure of Vector Quantizers},
  journal          = ITIT,
  year             = {1982},
  volume           = {28},
  number           = {2},
  pages            = {157--166}
}

@article{han2025balancekv,
  author           = {Han, Insu and Kapralov, Michael and Kochetkova, Ekaterina and Sheth, Kshiteej and Zandieh, Amir},
  title            = {{BalanceKV}: {KV} Cache Compression through Discrepancy Theory},
  journal          = ARXIV,
  year             = {2025}
}

@article{hooper2024kvquant,
  author           = {Hooper, Coleman and Kim, Sehoon and Mohammadzadeh, Hiva and Mahoney, Michael W. and Shao, Yakun Sophia and Keutzer, Kurt and Gholami, Amir},
  title            = {{KVQuant}: Towards 10 Million Context Length {LLM} Inference with {KV} Cache Quantization},
  journal          = ARXIV,
  year             = {2024}
}

@misc{kamradt2023needle,
  author           = {Kamradt, Greg},
  title            = {Needle in a Haystack --- Pressure Testing {LLMs}},
  year             = {2023},
  howpublished     = {\url{https://github.com/gkamradt/LLMTest_NeedleInAHaystack}}
}

@inproceedings{hsieh2024ruler,
  author           = {Hsieh, Cheng-Ping and Sun, Simeng and Kriman, Samuel and Acharya, Shantanu and Rekesh, Dima and Jia, Fei and Ginsburg, Boris},
  title            = {{RULER}: What's the Real Context Size of Your Long-Context Language Models?},
  booktitle        = {Proceedings of the Conference on Language Modeling (COLM)},
  year             = {2024}
}

@article{kang2024gear,
  author           = {Kang, Hao and Zhang, Qingru and Kundu, Souvik and Jeong, Geonhwa and Liu, Zaoxing and Krishna, Tushar and Zhao, Tuo},
  title            = {{GEAR}: An Efficient {KV} Cache Compression Recipe for Near-Lossless Generative Inference of {LLM}},
  journal          = ARXIV,
  year             = {2024}
}

@article{kim2024lexico,
  author           = {Kim, Junhyuck and Park, Jongho and Cho, Jaewoong and Papailiopoulos, Dimitris},
  title            = {{Lexico}: Extreme {KV} Cache Compression via Sparse Coding over Universal Dictionaries},
  journal          = ARXIV,
  year             = {2024}
}

@article{li2024snapkv,
  author           = {Li, Yuhong and Huang, Yingbing and Yang, Bowen and Venkitesh, Bharat and Locatelli, Acyr and Ye, Hanchen and Cai, Tianle and Lewis, Patrick and Chen, Deming},
  title            = {{SnapKV}: {LLM} Knows What You Are Looking For Before Generation},
  journal          = ARXIV,
  year             = {2024}
}

@article{lin2024awq,
  author           = {Lin, Ji and Tang, Jiaming and Tang, Haotian and Yang, Shang and Chen, Wei-Ming and Wang, Wei-Chen and Xiao, Guangxuan and Dang, Xingyu and Gan, Chuang and Han, Song},
  title            = {{AWQ}: Activation-Aware Weight Quantization for On-Device {LLM} Compression and Acceleration},
  journal          = {Proceedings of Machine Learning and Systems},
  year             = {2024},
  volume           = {6},
  pages            = {87--100}
}

@article{liu2024scissorhands,
  author           = {Liu, Zichang and Desai, Aditya and Liao, Fangshuo and Wang, Weitao and Xie, Victor and Xu, Zhaozhuo and Kyrillidis, Anastasios and Shrivastava, Anshumali},
  title            = {Scissorhands: Exploiting the Persistence of Importance Hypothesis for {LLM} {KV} Cache Compression at Test Time},
  journal          = NIPS,
  year             = {2024},
  volume           = {36}
}

@article{liu2024kivi,
  author           = {Liu, Zirui and Yuan, Jiayi and Jin, Hongye and Zhong, Shaochen and Xu, Zhaozhuo and Braverman, Vladimir and Chen, Beidi and Hu, Xia},
  title            = {{KIVI}: A Tuning-Free Asymmetric 2-Bit Quantization for {KV} Cache},
  journal          = ARXIV,
  year             = {2024}
}

@article{lloyd1982least,
  author           = {Lloyd, Stuart},
  title            = {Least Squares Quantization in {PCM}},
  journal          = ITIT,
  year             = {1982},
  volume           = {28},
  number           = {2},
  pages            = {129--137}
}

@article{max1960quantizing,
  author           = {Max, Joel},
  title            = {Quantizing for Minimum Distortion},
  journal          = {IRE Transactions on Information Theory},
  year             = {1960},
  volume           = {6},
  number           = {1},
  pages            = {7--12}
}

@article{panter1951quantization,
  author           = {Panter, Philip F. and Dite, Ward},
  title            = {Quantization Distortion in Pulse-Count Modulation with Nonuniform Spacing of Levels},
  journal          = {Proceedings of the IRE},
  year             = {1951},
  volume           = {39},
  number           = {1},
  pages            = {44--48}
}

@article{shah2024flashattention3,
  author           = {Shah, Jay and Bikshandi, Ganesh and Zhang, Ying and Thakkar, Vijay and Ramani, Pradeep and Dao, Tri},
  title            = {{FlashAttention-3}: Fast and Accurate Attention with Asynchrony and Low-Precision},
  journal          = ARXIV,
  year             = {2024}
}

@article{su2025rotatekv,
  author           = {Su, Zunhai and Chen, Zhe and Shen, Wang and Wei, Hanyu and Li, Linge and Yu, Huangqi and Yuan, Kehong},
  title            = {{RotateKV}: Accurate and Robust 2-Bit {KV} Cache Quantization for {LLMs} via Outlier-Aware Adaptive Rotations},
  journal          = ARXIV,
  year             = {2025}
}

@inproceedings{xiao2023smoothquant,
  author           = {Xiao, Guangxuan and Lin, Ji and Seznec, Mickael and Wu, Hao and Demouth, Julien and Han, Song},
  title            = {{SmoothQuant}: Accurate and Efficient Post-Training Quantization for Large Language Models},
  booktitle        = ICML,
  year             = {2023},
  pages            = {38087--38099}
}

@article{xiao2023streamingllm,
  author           = {Xiao, Guangxuan and Tian, Yuandong and Chen, Beidi and Han, Song and Lewis, Mike},
  title            = {Efficient Streaming Language Models with Attention Sinks},
  journal          = ARXIV,
  year             = {2023}
}

@article{yang2024notoken,
  author           = {Yang, June Yong and Kim, Byeongwook and Bae, Jeongin and Kwon, Beomseok and Park, Gunho and Yang, Eunho and Kwon, Se Jung and Lee, Dongsoo},
  title            = {No Token Left Behind: Reliable {KV} Cache Compression via Importance-Aware Mixed Precision Quantization},
  journal          = ARXIV,
  year             = {2024}
}

@article{yue2024wkvquant,
  author           = {Yue, Yuxuan and Yuan, Zhihang and Duanmu, Haojie and Zhou, Sifan and Wu, Jianlong and Nie, Liqiang},
  title            = {{WKVQuant}: Quantizing Weight and Key/Value Cache for Large Language Models Gains More},
  journal          = ARXIV,
  year             = {2024}
}

@phdthesis{zador1964development,
  author           = {Zador, Paul L.},
  title            = {Development and Evaluation of Procedures for Quantizing Multivariate Distributions},
  year             = {1964},
  school           = {Stanford University}
}

@article{zhang2024kv1bit,
  author           = {Zhang, Tianyi and Yi, Jonah and Xu, Zhaozhuo and Shrivastava, Anshumali},
  title            = {{KV} Cache Is 1 Bit Per Channel: Efficient Large Language Model Inference with Coupled Quantization},
  journal          = ARXIV,
  year             = {2024}
}

@inproceedings{zhang2018lpips,
  author           = {Zhang, Richard and Isola, Phillip and Efros, Alexei A. and Shechtman, Eli and Wang, Oliver},
  title            = {The Unreasonable Effectiveness of Deep Features as a Perceptual Metric},
  booktitle        = CVPR,
  year             = {2018}
}

@article{zhang2024h2o,
  author           = {Zhang, Zhenyu and Sheng, Ying and Zhou, Tianyi and Chen, Tianlong and Zheng, Lianmin and Cai, Ruisi and Song, Zhao and Tian, Yuandong and R{\'e}, Christopher and Barrett, Clark and others},
  title            = {{H2O}: Heavy-Hitter Oracle for Efficient Generative Inference of Large Language Models},
  journal          = NIPS,
  year             = {2024},
  volume           = {36}
}

@inproceedings{dao2022flashattention,
  author           = {Tri Dao and Daniel Y. Fu and Stefano Ermon and A. Rudra and Christopher R'e},
  title            = {{FlashAttention}: Fast and Memory-Efficient Exact Attention with {IO}-Awareness},
  booktitle        = NIPS,
  year             = {2022},
  doi              = {10.52202/068431-1189}
}

@article{liu2024lostinmiddle,
  author           = {Liu, Nelson F. and Lin, Kevin and Hewitt, John and Paranjape, Ashwin and Bevilacqua, Michele and Petroni, Fabio and Liang, Percy},
  title            = {Lost in the Middle: How Language Models Use Long Contexts},
  journal          = TACL,
  year             = {2024},
  volume           = {12},
  pages            = {157--173},
  doi              = {10.1162/tacl_a_00638}
}

@inproceedings{tillet2019triton,
  author           = {Tillet, Philippe and Kung, H. T. and Cox, David},
  title            = {{Triton}: An Intermediate Language and Compiler for Tiled Neural Network Computations},
  booktitle        = {Proceedings of the 3rd ACM SIGPLAN International Workshop on Machine Learning and Programming Languages},
  year             = {2019},
  doi              = {10.1145/3315508.3329973}
}

@inproceedings{zhao2024atom,
  author           = {Zhao, Yilong and Lin, Chien-Yu and Zhu, Kan and Ye, Zihao and Chen, Lequn and Zheng, Size and Ceze, Luis and Krishnamurthy, Arvind and Chen, Tianqi and Kasikci, Baris},
  title            = {{Atom}: Low-Bit Quantization for Efficient and Accurate {LLM} Serving},
  booktitle        = MLS,
  year             = {2024},
  volume           = {6},
  pages            = {196--209}
}

\clearpage
\appendix

\section{Encoder and decoder algorithms}
\label{app:algos}

\Algo{encode} gives the encoder as it is implemented: one pass per
key, with all intermediate state (rotated vector, triplet norms,
octahedral coordinates, integer indices) kept in registers. At decode
time, \Algo{decode} fuses bit unpacking, octahedral decode, centroid
gather, value dequantization, and online softmax into a single split-K
flash-decoding kernel in the style of
\citet{dao2022flashattention,shah2024flashattention3}. Keys are
reconstructed in registers on the fly, and the only per-step memory
traffic is the packed KV state and the running softmax statistics
$(m,\ell,\mathrm{acc})$. A reduction kernel merges the per-split
triples with the standard flash-decoding identity. These kernels are
an artefact of the implementation, not of the algorithm: every number
in \Section{results} can be reproduced with a pure PyTorch reference
of the same mathematical operations at proportionally higher
wall-clock cost.

\begin{algorithm}[H]
\caption{\ours encoder with joint $(\xi,\eta,\rho)$ rounding
(\sect{method:rounding}), one program per key. Line~14 is the
3$\times$3 optimal-rounding refinement; setting the candidate set
$\Delta$ to $\{(0,0)\}$ recovers the legacy scalar-rounding baseline.}
\label{alg:encode}
\begin{algorithmic}[1]
\Require $\vect{k}\in\real^{d}$, sign vector $\vect{s}\in\{\pm1\}^{d}$,
    codebook centroids $\mathcal{C}_{\xi}\in\real^{2^{b_{\mathrm{dir}}}}$,
    $\mathcal{C}_{\rho}\in\real^{2^{b_{\mathrm{nrm}}}}$,
    boundaries $\mathcal{B}_{\xi},\mathcal{B}_{\rho}$ (midpoints of adjacent centroids),
    candidate set $\Delta\subseteq\{-1,0,1\}^{2}$ (default $\Delta=\{-1,0,1\}^{2}$).
\State $\gamma\gets\sqrt{\sum_{i=1}^{d} k_i^{2}}$ \Comment{fp32 norm}
\State $\vect{\tilde u}\gets \vect{k}/\max(\gamma,\epsilon)$
\State $\vect{u}\gets \matt{H}\,(\vect{s}\odot\vect{\tilde u})$
    \Comment{in-register WHT butterfly with normalized $\matt{H}$}
\For{$i = 0$ \textbf{to} $n_{\mathrm{tri}}-1$}
    \State $\vect{t}_i\gets \vect{u}_{3i:3i+3}$;\quad $(x,y,z)\gets\vect{t}_i$
    \State $(p_x,p_y,p_z)\gets (x,y,z)/\max(|x|+|y|+|z|,\epsilon)$
    \If{$p_z \geq 0$}
        \State $(\xi_i,\eta_i)\gets (p_x,p_y)$
    \Else
        \State $(\xi_i,\eta_i)\gets \bigl(\mathrm{sign}(p_x)(1-|p_y|),\;\mathrm{sign}(p_y)(1-|p_x|)\bigr)$
    \EndIf
    \State $j_x\gets \mathrm{searchsorted}(\mathcal{B}_{\xi},\xi_i)$;\;
           $j_y\gets \mathrm{searchsorted}(\mathcal{B}_{\xi},\eta_i)$ \Comment{scalar seed}
    \State $s^{*}\gets -\infty$;\;\; $(J_x,J_y)\gets(j_x,j_y)$
    \For{$(\delta_x,\delta_y)\in\Delta$}
        \State $j'_x\gets\mathrm{clip}(j_x+\delta_x,0,2^{b_{\mathrm{dir}}}-1)$;\;
               $j'_y\gets\mathrm{clip}(j_y+\delta_y,0,2^{b_{\mathrm{dir}}}-1)$
        \State $s\gets \vect{t}_i^{\!\top}\,\mathrm{Oct}^{-1}(\mathcal{C}_{\xi}[j'_x],\mathcal{C}_{\xi}[j'_y])$
            \Comment{\eqn{oct-decode}}
        \State \textbf{if } $s>s^{*}$ \textbf{then} $(s^{*},J_x,J_y)\gets(s,j'_x,j'_y)$
    \EndFor
    \State $I^{\mathrm{dir}}_{i,0}\gets J_x$;\; $I^{\mathrm{dir}}_{i,1}\gets J_y$
    \State $I^{\mathrm{nrm}}_{i} \gets \mathrm{searchsorted}(\mathcal{B}_{\rho},\,\mathrm{clip}(s^{*},0,1))$
\EndFor
\State \Return $\bigl(\gamma,\;\mathrm{pack}(\{I^{\mathrm{dir}}\},\,b_{\mathrm{dir}}),
                    \;\mathrm{pack}(\{I^{\mathrm{nrm}}\},\,b_{\mathrm{nrm}})\bigr)$
\end{algorithmic}
\end{algorithm}

\begin{algorithm}[H]
\caption{\ours split-K flash decode, one program per (head, split).}
\label{alg:decode}
\begin{algorithmic}[1]
\Require $\vect{q}_{\mathrm{rot}}\in\real^{d}$,
    $(\gamma_{t},\mathcal{I}_{\mathrm{dir},t},\mathcal{I}_{\mathrm{nrm},t})_{t\in\mathrm{chunk}}$,
    $V$ codec state, codebook centroids $\vect{c}_{\xi}$, $\vect{c}_{\rho}$.
\State $(m,\ell,\mathrm{acc}) \gets (-\infty,\,0,\,\vect{0})$
\For{$t$ \textbf{in} chunk}
    \State Load packed bytes $\mathcal{I}_{\mathrm{dir},t},\mathcal{I}_{\mathrm{nrm},t}$ and $\gamma_t$.
    \State Extract $(I^{\mathrm{dir}}_{t,i,0},I^{\mathrm{dir}}_{t,i,1},I^{\mathrm{nrm}}_{t,i})_{i=0}^{n_{\mathrm{tri}}-1}$ via shift+mask.
    \State Gather $(\hat{\xi}_{t,i},\hat{\eta}_{t,i},\hat{\rho}_{t,i})$ from $(\vect{c}_{\xi},\vect{c}_{\rho})$.
    \State $\hat{\vect{n}}_{t,i}\gets\mathrm{Oct}^{-1}(\hat{\xi}_{t,i},\hat{\eta}_{t,i})$.
    \State $s_t\gets \gamma_t\cdot\sum_i \hat{\rho}_{t,i}\,\vect{q}_{\mathrm{rot},i}^{\!\top}\hat{\vect{n}}_{t,i}$.
    \State $(m,\ell,\mathrm{acc})\gets\mathrm{onlineSoftmax}(m,\ell,\mathrm{acc},s_t/\sqrt{d},\hat{\vect{v}}_t)$.
\EndFor
\State Emit partial $(m,\ell,\mathrm{acc})$ to split buffer.
\end{algorithmic}
\end{algorithm}

\section{Mathematical details}
\label{app:proofs}

\begin{proposition}[Inner-product invariance]
\label{prop:invariance}
For any $\vect{q},\vect{k}\in\real^{d}$ and $\vect{s}\in\{\pm 1\}^{d}$,
$\vect{q}^{\!\top}\vect{k}=(\matt{R}\vect{q})^{\!\top}(\matt{R}\vect{k})
=\gamma\,(\matt{R}\vect{q})^{\!\top}\vect{u}$.
\end{proposition}
\noindent
Any quantization of $\vect{u}$ at decode time can therefore be combined
with a matching rotation of the query, and the dot product is unbiased
in expectation.

\begin{proposition}[Marginal concentration
{\citep[Thm.~1]{zandieh2025turboquant}}]
\label{prop:marginal}
If $\vect{\tilde u}$ is uniformly distributed on $S^{d-1}$ and
$\vect{s}$ is independent uniform on $\{\pm1\}^{d}$, each coordinate
$u_i$ of $\vect{u}=\matt{R}\vect{\tilde u}$ has the symmetric-Beta
density of \eqn{beta-marginal}.
\end{proposition}

\begin{proof}[Proof of \cref{lem:triplet-norm}]
Let \(\vect{g}\sim\mathcal N(\vect{0},\matt{I}_d)\) and set
\(\vect{u}=\vect{g}/\|\vect{g}\|_2\). Then \(\vect{u}\) is uniform on
\(S^{d-1}\). Since the random-signed WHT \(\matt{R}\) is orthogonal, it
preserves this distribution, so the rotated direction \(\matt{R}\vect{u}\)
has the same law as \(\vect{u}\). It therefore suffices to consider any full
three-coordinate block of \(\vect{u}\).

For such a triplet,
\[
\rho_i^2
=
u_{3i+1}^2+u_{3i+2}^2+u_{3i+3}^2
=
\frac{A}{A+B},
\]
where
\[
A=\sum_{j=3i+1}^{3i+3} g_j^2\sim\chi^2_3,
\qquad
B=\sum_{j\notin\{3i+1,3i+2,3i+3\}} g_j^2\sim\chi^2_{d-3}.
\]
The variables \(A\) and \(B\) are independent because they are sums over
disjoint Gaussian coordinates. Since
\(\chi^2_k=\mathrm{Gamma}(k/2,2)\), the standard gamma-ratio identity gives
\[
\rho_i^2
=
\frac{A}{A+B}
\sim
\mathrm{Beta}\!\left(\frac32,\frac{d-3}{2}\right).
\]
Finally, applying the change of variables \(x=\rho_i^2\), \(dx=2\rho_i\,
d\rho_i\), yields
\[
f_{\rho}(r)
=
\frac{2r^2(1-r^2)^{(d-5)/2}}
{B\!\left(\frac32,\frac{d-3}{2}\right)},
\qquad r\in[0,1].
\]
\end{proof}

\section{Derivations}

\subsection{Magnitude-Direction Split (Equation~\ref{eq:split})}
For any nonzero key $\vect{k}\in\real^d$, set
\[
    \gamma=\norm{\vect{k}}_2,\qquad
    \vect{\tilde u}=\vect{k}/\gamma .
\]
Then
\[
    \norm{\vect{\tilde u}}_2
    =\frac{\norm{\vect{k}}_2}{\gamma}=1,
\]
so $\vect{\tilde u}\in S^{d-1}$ and
$\vect{k}=\gamma\vect{\tilde u}$. The implementation stores the original
norm $\gamma$ and divides by a clamped positive denominator only to handle
zero or tiny keys safely.

\subsection{Sign-Flipped WHT Rotation (Equation~\ref{eq:rotation})}
Let $\matt{D}_{\vect{s}}=\mathrm{diag}(\vect{s})$ with
$s_i\in\{\pm1\}$, and let $\matt{H}$ be the normalized Hadamard matrix. The
implemented rotation is
\[
    \matt{R}=\matt{H}\matt{D}_{\vect{s}},
    \qquad \vect{u}=\matt{R}\vect{\tilde u}.
\]
Because $\matt{H}^{\top}\matt{H}=\matt{I}$ and
$\matt{D}_{\vect{s}}^{\top}\matt{D}_{\vect{s}}=\matt{I}$,
\[
    \matt{R}^{\top}\matt{R}
    =\matt{D}_{\vect{s}}^{\top}\matt{H}^{\top}\matt{H}\matt{D}_{\vect{s}}
    =\matt{I}.
\]
Thus $\matt{R}$ maps $S^{d-1}$ to itself. Since the normalized Hadamard is
self-inverse, $\matt{R}^{-1}=\matt{R}^{\top}=\matt{D}_{\vect{s}}\matt{H}$.
The same identity gives the rotated-frame inner product
\[
    \vect{q}^{\!\top}\vect{k}
    =\gamma\,\vect{q}^{\!\top}\matt{R}^{\top}\vect{u}
    =\gamma\,(\matt{R}\vect{q})^{\!\top}\vect{u}.
\]

\subsection{Coordinate Marginal (Equation~\ref{eq:beta-marginal})}
For $\vect{u}\sim\mathrm{Unif}(S^{d-1})$, one coordinate $U$ satisfies
$(U+1)/2\sim\mathrm{Beta}(a,a)$ with $a=(d-1)/2$. With
$z=(u+1)/2$ and $dz/du=1/2$,
\[
\begin{aligned}
f_U(u)
&=\frac{1}{2B(a,a)}
  \left(\frac{1+u}{2}\right)^{a-1}
  \left(\frac{1-u}{2}\right)^{a-1} \\
&=
\frac{(1-u^2)^{(d-3)/2}}
     {B\!\left((d-1)/2,(d-1)/2\right)\,2^{d-2}},
\qquad u\in[-1,1].
\end{aligned}
\]
This is the density implemented by the scalar MSE Lloyd-Max codebook. For
\ours, it mainly justifies the rotated-sphere prior; the default codec uses
the triplet and octahedral marginals below.

\subsection{Triplet Split (Section~\ref{sec:method:triplets})}
Pad $\vect{u}$ with zeros to length $3n_{\mathrm{tri}}$ and split the padded
vector into contiguous triplets $\vect{t}_i$. Since padding adds only zeros,
\[
    \sum_i \norm{\vect{t}_i}_2^2=\norm{\vect{u}}_2^2=1.
\]
Therefore $0\leq\rho_i=\norm{\vect{t}_i}_2\leq1$. If $\rho_i>0$, then
\[
    \vect{n}_i=\vect{t}_i/\rho_i,\qquad
    \norm{\vect{n}_i}_2=1,
\]
so $\vect{n}_i\in S^2$ and $\vect{t}_i=\rho_i\vect{n}_i$. When
$\rho_i=0$, the direction is mathematically arbitrary; the reference and
Triton encoders use an $\epsilon$-safe divisor to produce a finite placeholder.

\subsection{Triplet-Norm Marginal (Equation~\ref{eq:rho-pdf})}
Generate a uniform sphere point as $\vect{u}=\vect{z}/\norm{\vect{z}}_2$ with
$z_j\stackrel{\mathrm{iid}}{\sim}\mathcal{N}(0,1)$. For any three-coordinate
triplet,
\[
    \rho_i^2
    =\frac{z_1^2+z_2^2+z_3^2}{\sum_{j=1}^{d}z_j^2}
    =\frac{X}{X+Y},
\]
where $X\sim\chi^2_3$, $Y\sim\chi^2_{d-3}$, and $X,Y$ are independent. Hence
\[
    \rho_i^2\sim\mathrm{Beta}\!\left(\frac{3}{2},\frac{d-3}{2}\right).
\]
Let $S=\rho_i^2$ and $R=\rho_i$. The change of variables $s=r^2$ gives
$ds/dr=2r$, so
\[
\begin{aligned}
f_R(r)
&=f_S(r^2)\,2r \\
&=\frac{2r(r^2)^{1/2}(1-r^2)^{(d-5)/2}}
        {B(3/2,(d-3)/2)} \\
&=\frac{2r^2(1-r^2)^{(d-5)/2}}
        {B(3/2,(d-3)/2)},\qquad r\in[0,1].
\end{aligned}
\]
This is the density integrated by the implemented triplet-norm Lloyd-Max
codebook.

\subsection{Octahedral Encode (Equation~\ref{eq:oct-encode})}
For $\vect{n}=(x,y,z)\in S^2$, define
$\ell=|x|+|y|+|z|$ and $\vect{p}=\vect{n}/\ell$. Then
$|p_x|+|p_y|+|p_z|=1$, so $\vect{p}$ lies on the unit $L_1$ octahedron. If
$p_z\geq0$, the upper face is already parameterized by $(p_x,p_y)$. If
$p_z<0$, the lower face is folded into the square by
\[
    (\xi,\eta)
    =\left(\mathrm{sign}(p_x)(1-|p_y|),
           \mathrm{sign}(p_y)(1-|p_x|)\right).
\]
On the fold boundary $p_z=0$, the two branches agree because
$|p_x|+|p_y|=1$. The implementation uses the convention
$\mathrm{sign}(0)=+1$ and clamps denominators away from zero.

\subsection{Octahedral Decode (Equation~\ref{eq:oct-decode})}
Given $(\xi,\eta)\in[-1,1]^2$, set $r=1-|\xi|-|\eta|$. If $r\geq0$, the
point is on an unfolded upper face and the unnormalized vector is
$(\xi,\eta,r)$. If $r<0$, inverting the fold gives
\[
    \xi'=\mathrm{sign}(\xi)(1-|\eta|),\qquad
    \eta'=\mathrm{sign}(\eta)(1-|\xi|).
\]
Thus, with $(\xi',\eta')=(\xi,\eta)$ in the $r\geq0$ branch,
\[
    \vect{n}(\xi,\eta)
    =\frac{(\xi',\eta',r)}
          {\norm{(\xi',\eta',r)}_2}.
\]
The final normalization maps the octahedral surface back to $S^2$; both the
reference decoder and the fused attention kernel implement this formula on
dequantized oct-coordinate centroids.

\subsection{Octahedral Marginal (Equation~\ref{eq:oct-marginal})}
For the implemented fold, the induced density on the square is
obtained by composing the affine octahedral inverse with radial projection
onto $S^2$. On each triangular fold region, write the unnormalized inverse as
$\vect{q}(\xi,\eta)$. Uniform surface measure pulls back to
\[
    g(\xi,\eta)=\frac{1}{4\pi}\norm{\vect{q}(\xi,\eta)}_2^{-3},
\]
because each affine face has unit volume factor in
$|\det(\vect{q},\partial_\xi\vect{q},\partial_\eta\vect{q})|$. By symmetry
the marginal depends only on $a=|\xi|$. Integrating the two inner and two
folded vertical segments gives
\[
f_{\xi}(\xi)
=\frac{1}{2\pi}\left[
\int_0^{1-a}\!\!
\frac{ds}{\left(a^2+s^2+(1-a-s)^2\right)^{3/2}}
+
\int_0^a\!\!
\frac{ds}{\left(s^2+(1-a)^2+(s-a)^2\right)^{3/2}}
\right],
\]
and evaluating the integrals yields
\[
    f_{\xi}(\xi)=
    \frac{1}{\pi\sqrt{a^2+(1-a)^2}}
    \left(
      \frac{1-a}{1-2a+3a^2}
      + \frac{a}{2-4a+3a^2}
    \right),
    \qquad a=|\xi|.
\]
The implementation does not evaluate this expression directly; it samples
$\vect{n}\sim\mathrm{Unif}(S^2)$, applies the same octahedral fold, and trains
a one-dimensional Lloyd-Max codebook on the empirical coordinate marginal.

\subsection{Triplet MSE Bound (Equation~\ref{eq:triplet-mse})}
With $\vect{t}_i=\rho_i\vect{n}_i$ and
$\hat{\vect{t}}_i=\hat{\rho}_i\hat{\vect{n}}_i$, add and subtract
$\rho_i\hat{\vect{n}}_i$:
\[
\begin{aligned}
\norm{\vect{t}_i-\hat{\vect{t}}_i}_2^2
&=\norm{\rho_i(\vect{n}_i-\hat{\vect{n}}_i)
      +(\rho_i-\hat{\rho}_i)\hat{\vect{n}}_i}_2^2 \\
&\leq
2\rho_i^2\norm{\vect{n}_i-\hat{\vect{n}}_i}_2^2
+2(\rho_i-\hat{\rho}_i)^2\norm{\hat{\vect{n}}_i}_2^2 \\
&=
2(\rho_i-\hat{\rho}_i)^2
+2\rho_i^2\norm{\vect{n}_i-\hat{\vect{n}}_i}_2^2 .
\end{aligned}
\]
The last line uses $\norm{\hat{\vect{n}}_i}_2=1$, enforced by octahedral
decode normalization.

\subsection{Expected MSE Budget (Equation~\ref{eq:mse-budget})}
Under the Gaussian sphere construction above, the block radius and direction
are independent: the Gaussian block direction is uniform on $S^2$, while the
block and complement radii determine $\rho_i$. Also
\[
    \expect[\rho_i^2]
    =\frac{3/2}{3/2+(d-3)/2}
    =\frac{3}{d}.
\]
Taking expectations in the triplet bound and using high-rate scalar
Lloyd-Max/Panter-Dite scaling,
\[
    \expect[(\rho_i-\hat{\rho}_i)^2]
    \approx C_\rho\sigma_\rho^2 4^{-b_{\mathrm{nrm}}},
    \qquad
    \expect[\norm{\vect{n}_i-\hat{\vect{n}}_i}_2^2]
    \approx C_n\sigma_{\vect{n}}^2 4^{-b_{\mathrm{dir}}},
\]
gives
\[
\begin{aligned}
\expect\!\left[\norm{\vect{t}_i-\hat{\vect{t}}_i}_2^2\right]
&\approx
2C_\rho\sigma_\rho^2 4^{-b_{\mathrm{nrm}}}
+2\expect[\rho_i^2]C_n\sigma_{\vect{n}}^2 4^{-b_{\mathrm{dir}}} \\
&=
2C_\rho\sigma_\rho^2 4^{-b_{\mathrm{nrm}}}
+(6/d)C_n\sigma_{\vect{n}}^2 4^{-b_{\mathrm{dir}}}.
\end{aligned}
\]
Here $\sigma_\rho^2$ denotes the variance of the scalar norm $\rho$, which is
what the implementation quantizes.

\subsection{Lagrangian Bit Allocation (Equation~\ref{eq:bitsplit-opt})}
Let
\[
    A=2C_\rho\sigma_\rho^2,\qquad
    D=\frac{6}{d}C_n\sigma_{\vect{n}}^2.
\]
The high-rate objective is
\[
    f(b_{\mathrm{nrm}},b_{\mathrm{dir}})
    =A4^{-b_{\mathrm{nrm}}}+D4^{-b_{\mathrm{dir}}},
    \qquad b_{\mathrm{nrm}}+2b_{\mathrm{dir}}=B_{\mathrm{tri}}.
\]
For the Lagrangian:
\[
    \mathcal{L}
    =A4^{-b_{\mathrm{nrm}}}+D4^{-b_{\mathrm{dir}}}
     +\lambda(b_{\mathrm{nrm}}+2b_{\mathrm{dir}}-B_{\mathrm{tri}}),
\]
stationarity (deriving by $b_{\mathrm{nrm}}$ and $b_{\mathrm{dir}}$ and equating to 0) gives
\[
    -A\log_4\,4^{-b_{\mathrm{nrm}}}+\lambda=0,
    \qquad
    -D\log_4\,4^{-b_{\mathrm{dir}}}+2\lambda=0.
\]
Thus
\[
    A4^{-b_{\mathrm{nrm}}}
    =\frac{D}{2}4^{-b_{\mathrm{dir}}},
\]
and
\[
    b_{\mathrm{dir}}^\star-b_{\mathrm{nrm}}^\star
    =\log_4\!\left(\frac{D}{2A}\right)
    =\log_4\!\left(
        \frac{3C_n\sigma_{\vect{n}}^2}
             {2dC_\rho\sigma_\rho^2}\right).
\]
The factor of two in the denominator is the shadow price of spending one
additional bit on each of two octahedral coordinates.

\subsection{Bit-Gap Scaling}
The squared triplet norm has
$\mathrm{Var}(\rho_i^2)=\mathcal{O}(d^{-2})$, but the implemented norm
codebook quantizes $\rho_i$ itself. Since
$\rho_i=\chi_3/\sqrt{d}+o(d^{-1/2})$ under the sphere prior,
\[
    \sigma_\rho^2=\mathrm{Var}(\rho_i)=\mathcal{O}(d^{-1}),
    \qquad
    \sigma_{\vect{n}}^2=\mathcal{O}(1).
\]
Substituting these scalings into the corrected stationarity condition gives
\[
    b_{\mathrm{dir}}^\star-b_{\mathrm{nrm}}^\star=\mathcal{O}(1),
\]
up to constants from the source densities and the octahedral metric. The
implemented $(b+1,b-1)$ split is therefore a finite-dimensional codebook
choice supported by the empirical sweeps, not a consequence of a growing
asymptotic gap for direct $\rho$ quantization.

\subsection{Joint Rounding (Equation~\ref{eq:triplet-joint-expand})}
\begin{align}\ell(\hat{\xi}_i,\hat{\eta}_i,\hat{\rho}_i)
&= ||\boldsymbol{t}_i - \hat{\rho}_i \boldsymbol{n}(\hat{\xi}_i, \hat{\eta}_i)||_2^2 \\
&= \rho_i^2 ||\boldsymbol{n}(\xi_i, \eta_i)||_2^2 - 2 \hat{\rho}_i \vect{t}_i^\top \boldsymbol{n}(\hat{\xi}_i, \hat{\eta}_i) + \hat{\rho}_i^2 || \boldsymbol{n}(\hat{\xi}_i, \hat{\eta}_i)||_2^2 \\
&= \rho_i^2 - 2 \hat{\rho}_i s_i(\xi_i, \eta_i)  + \hat{\rho}_i^2
\end{align}
where $||\boldsymbol{n}(\xi_i, \eta_i)||_2^2 = \boldsymbol{n}(\hat{\xi}_i, \hat{\eta}_i)||_2^2 = 1$, and $s_i(\hat{\xi}_i, \hat{\eta}_i) = \vect{t}_i^\top \boldsymbol{n}(\hat{\xi}_i, \hat{\eta}_i)$ i.e. the dot product of the true vector and the quantized direction candidate. This means minimum of this parabola on $\hat{\rho}_i$ occurs exactly when $\hat{\rho}_i = s_i$. Therefore, the best quantized radius is NOT the centroid nearest to the true radius ($\rho_i$), but the centroid nearest to the projected dot product ($s_i$).

\subsection{Score Factorization (Equation~\ref{eq:score-exact})}
The decoded key has
$\hat{\vect{k}}=\gamma\matt{R}^{\top}\hat{\vect{u}}$, so
\[
    \vect{q}^{\!\top}\hat{\vect{k}}
    =\gamma\,(\matt{R}\vect{q})^{\!\top}\hat{\vect{u}}
    =\gamma\,\vect{q}_{\mathrm{rot}}^{\!\top}\hat{\vect{u}}.
\]
The decoder reconstructs each rotated triplet as
$\hat{\vect{u}}_i=\hat{\rho}_i\hat{\vect{n}}_i$, where
\[
    \hat{\vect{n}}_i
    =\mathrm{Oct}^{-1}\!\left(
        \mathcal{C}_{\xi}[I^{\mathrm{dir}}_{i,0}],
        \mathcal{C}_{\xi}[I^{\mathrm{dir}}_{i,1}]
      \right),
    \qquad
    \hat{\rho}_i=\mathcal{C}_{\rho}[I_i^{\mathrm{nrm}}].
\]
Therefore
\[
    \vect{q}^{\!\top}\hat{\vect{k}}
    =\gamma\sum_{i=0}^{n_{\mathrm{tri}}-1}
      \hat{\rho}_i\,\vect{q}_{\mathrm{rot},i}^{\!\top}\hat{\vect{n}}_i.
\]
The fused attention kernel applies the usual attention scale after this raw
score is formed.

\subsection{QJL Residual Estimator}
Let $\vect{r}=\vect{u}-\hat{\vect{u}}$ and
$\gamma_r=\norm{\vect{r}}_2$. For an independent ideal QJL projection
$\matt{R}'$, store
\[
    \vect{\sigma}=\mathrm{sign}(\matt{R}'\vect{r})\in\{\pm1\}^d .
\]
For $a=\vect{q}_{\mathrm{rot}}$, the asymmetric one-bit estimator is
\[
    \hat{z}(a,\vect{r})
    =\sqrt{\frac{\pi}{2d}}\,
      \gamma_r\,(\matt{R}'a)^{\!\top}\mathrm{sign}(\matt{R}'\vect{r}).
\]
Under the ideal QJL model,
\[
    \expect\!\left[(\matt{R}'a)_j
    \mathrm{sign}((\matt{R}'\vect{r})_j)\right]
    =\sqrt{\frac{2}{\pi d}}\,
      \frac{a^{\top}\vect{r}}{\gamma_r}.
\]
Summing over $d$ coordinates and multiplying by
$\sqrt{\pi/(2d)}\,\gamma_r$ gives
\[
    \expect[\hat{z}(a,\vect{r})]=a^{\top}\vect{r}.
\]
Since
$\vect{q}^{\!\top}\vect{k}
=\gamma\,\vect{q}_{\mathrm{rot}}^{\!\top}(\hat{\vect{u}}+\vect{r})$,
the corrected score is
\[
    \vect{q}^{\!\top}\hat{\vect{k}}
    +\gamma\,\hat{z}(\vect{q}_{\mathrm{rot}},\vect{r}).
\]
The implementation uses the same scaling with a structured WHT projection and
stores $\gamma_r$ in fp16, so exact unbiasedness is the ideal-model statement.

\section{Bit-allocation sweep}
\label{app:bit-split-sweep}

We sweep the diagonal $(b{+}\delta,b{-}\delta)$,
$\delta\in\{-2,-1,0,+1,+2\}$, around each uniform reference
$b\in\{2,3,4\}$ on $n{=}8192$ random Gaussian keys at $d{=}128$,
averaged over $4$ rotation seeds. \Table{bit-split-diagonal}
is the diagonal sweep with every entry expressed relative to the
uniform $(b,b)$ baseline. 

\begin{table}[t]
\centering
\titlecaption{Diagonal bit-split sweep $(b{+}\delta,b{-}\delta)$,
relative to the uniform $(b,b)$ reference}{Synthetic Gaussian keys,
$d{=}128$, $n{=}8192$, $4$ seeds. ``$\Delta$ MSE'' and
``$\Delta(1{-}\cos)$'' are the percentage change against the uniform
$(b,b)$ baseline at the same $b$; negative means the off-diagonal
split improves on uniform. ``invalid'' marks diagonal endpoints with a
zero-bit component (no codebook). The implemented split sits at
$\delta{=}{+}1$ for every $b\in\{2,3,4\}$ and is the unique diagonal
step that reduces MSE versus uniform.}
\label{tab:bit-split-diagonal}
\small
\setlength{\tabcolsep}{5pt}
\begin{tabular}{cc|cc|rrrr}
$b$ & $\delta$ & $b_{\mathrm{dir}}$ & $b_{\mathrm{nrm}}$ &
MSE & $1{-}\cos$ &
$\Delta$ MSE vs $(b,b)$ & $\Delta(1{-}\cos)$ vs $(b,b)$ \\
\midrule
2 & $-1$ & 1 & 3 & 0.4555 & 0.2628 & $+223.3\%$  & $+260.3\%$ \\
2 & $\phantom{-}0$ & 2 & 2 & 0.1409 & 0.0729 & $\phantom{+}0.0\%$ (ref) & $\phantom{+}0.0\%$ (ref) \\
\textbf{2} & $\boldsymbol{+1}$ & \textbf{3} & \textbf{1} & \textbf{0.0831} & \textbf{0.0421} & $\boldsymbol{-41.0\%}$ & $\boldsymbol{-42.4\%}$ \\
\midrule
3 & $-2$ & 1 & 5 & 0.4501 & 0.2593 & $+1104.8\%$ & $+1279.8\%$ \\
3 & $-1$ & 2 & 4 & 0.1273 & 0.0658 & $+240.6\%$  & $+250.2\%$ \\
3 & $\phantom{-}0$ & 3 & 3 & 0.0374 & 0.0188 & $\phantom{+}0.0\%$ (ref) & $\phantom{+}0.0\%$ (ref) \\
\textbf{3} & $\boldsymbol{+1}$ & \textbf{4} & \textbf{2} & \textbf{0.0243} & \textbf{0.0120} & $\boldsymbol{-35.0\%}$ & $\boldsymbol{-36.1\%}$ \\
3 & $+2$ & 5 & 1 & 0.0537 & 0.0268 & $+43.8\%$   & $+42.8\%$ \\
\midrule
4 & $-2$ & 2 & 6 & 0.1262 & 0.0653 & $+1217.1\%$ & $+1265.3\%$ \\
4 & $-1$ & 3 & 5 & 0.0332 & 0.0168 & $+246.8\%$  & $+250.5\%$ \\
4 & $\phantom{-}0$ & 4 & 4 & 0.0096 & 0.0048 & $\phantom{+}0.0\%$ (ref) & $\phantom{+}0.0\%$ (ref) \\
\textbf{4} & $\boldsymbol{+1}$ & \textbf{5} & \textbf{3} & \textbf{0.0067} & \textbf{0.0033} & $\boldsymbol{-30.5\%}$ & $\boldsymbol{-31.7\%}$ \\
4 & $+2$ & 6 & 2 & 0.0166 & 0.0081 & $+72.9\%$   & $+69.7\%$ \\
\end{tabular}
\end{table}

\noindent

\section{Rounding ablation}
\label{app:rounding-ablation}

\Table{rounding-ablation} reports the four rounding modes supported by
the reference encoder (\sect{method:rounding}) at matched bits. All
modes share the same bitstream and decoder, so this is a pure
encoder ablation. Metrics are averaged over five seeds on $n{=}4096$
random Gaussian keys at $d{=}128$ with $n_{\mathrm{query}}{=}64$.
\emph{tail95} is the $95$th-percentile per-key squared reconstruction
error. The full-codebook direction search is the joint-optimum upper
bound reachable with this $(\mathcal{C}_{\xi},\mathcal{C}_{\rho})$
pair; \texttt{local\_3x3} is the implemented default of \Algo{encode}.

\begin{table}[t]
\centering
\titlecaption{\ours rounding-mode ablation}{%
Scalar rounding vs.\ the joint-rounding variants of
\sect{method:rounding}. ``$\Delta$ tail95'' is the percentage change
in the $95$th-percentile squared error relative to the scalar
baseline. The 3$\times$3 local search is byte-identical to the
full direction search at every bit width tested.}
\label{tab:rounding-ablation}
\small
\setlength{\tabcolsep}{4pt}
\begin{tabular}{ll|rrrr|rr}
$b$ & rounding & $\cos$ $\uparrow$ & MSE $\downarrow$ & tail95 $\downarrow$ & $|\text{ip err}|$ $\downarrow$ & $\Delta$ MSE & $\Delta$ tail95 \\
\midrule
1 & scalar                & 0.692 & 0.6022 & 0.7935 & 7.065 & --      & -- \\
1 & local\_2$\times$2     & 0.703 & 0.5172 & 0.6664 & 6.554 & $-14.1\%$ & $-16.0\%$ \\
1 & \textbf{local\_3$\times$3}     & \textbf{0.703} & \textbf{0.5172} & \textbf{0.6664} & \textbf{6.554} & $\boldsymbol{-14.1\%}$ & $\boldsymbol{-16.0\%}$ \\
1 & full                  & 0.703 & 0.5172 & 0.6664 & 6.554 & $-14.1\%$ & $-16.0\%$ \\
\midrule
2 & scalar                & 0.955 & 0.0897 & 0.1205 & 2.722 & --      & -- \\
2 & local\_2$\times$2     & 0.957 & 0.0858 & 0.1152 & 2.662 & $-4.4\%$  & $-4.3\%$ \\
2 & \textbf{local\_3$\times$3}     & \textbf{0.958} & \textbf{0.0832} & \textbf{0.1119} & \textbf{2.620} & $\boldsymbol{-7.2\%}$  & $\boldsymbol{-7.1\%}$ \\
2 & full                  & 0.958 & 0.0832 & 0.1119 & 2.620 & $-7.2\%$  & $-7.1\%$ \\
\midrule
3 & scalar                & 0.987 & 0.0261 & 0.0365 & 1.464 & --      & -- \\
3 & local\_2$\times$2     & 0.988 & 0.0250 & 0.0351 & 1.433 & $-4.0\%$  & $-3.7\%$ \\
3 & \textbf{local\_3$\times$3}     & \textbf{0.988} & \textbf{0.0243} & \textbf{0.0343} & \textbf{1.414} & $\boldsymbol{-6.6\%}$  & $\boldsymbol{-5.9\%}$ \\
3 & full                  & 0.988 & 0.0243 & 0.0343 & 1.414 & $-6.6\%$  & $-5.9\%$ \\
\midrule
4 & scalar                & 0.997 & 0.0071 & 0.0102 & 0.763 & --      & -- \\
4 & local\_2$\times$2     & 0.997 & 0.0068 & 0.0098 & 0.749 & $-3.8\%$  & $-3.6\%$ \\
4 & \textbf{local\_3$\times$3}     & \textbf{0.997} & \textbf{0.0067} & \textbf{0.0096} & \textbf{0.739} & $\boldsymbol{-6.1\%}$  & $\boldsymbol{-5.7\%}$ \\
4 & full                  & 0.997 & 0.0067 & 0.0096 & 0.739 & $-6.1\%$  & $-5.7\%$ \\
\end{tabular}
\end{table}

\noindent
Takeaways consistent with the tangent-frame survey of
\citet{kapoulkine2026tangent}: optimal rounding mostly shifts the
encoder, not the codebook; the gain is largest at the tightest bit
budgets (where one misrounding consumes a large fraction of the
remaining precision); and a small local neighborhood suffices in
practice. The implemented 3$\times$3 search gives a uniform
$6$--$14\%$ MSE reduction at matched bit rate with \emph{zero} change
to the bitstream format or the decoder. Because only the encoder is
affected, previously serialised scalar-rounded states remain valid; the
joint-rounding improvement applies to every new \ours state produced
under the default encoder, including the \ours-QJL variant whose
residual stage sees a correspondingly smaller Stage-A error.

\section{QJL effective-rate accounting}
\label{app:qjl-accounting}

\Table{qjl_accounting} spells out the effective-rate cost of adding the one-bit residual side-car described in \sect{method:score}.

\begin{table}[t]
\centering
\titlecaption{QJL effective-rate accounting}{Effective bits per KV scalar (analytically computed from the codec configuration: nominal $K/V$ bits + $g{=}32$/16 group metadata + 1-bit JL residual + scale storage) compared at matched nominal bits. OCTOPUS-QJL adds a 1-bit JL residual on top of OCTOPUS, raising the effective rate by exactly $0.5$\,bits per scalar and giving up roughly that much reconstruction headroom inside the standard \emph{dequantize-then-dot} attention path. We therefore recommend the QJL variant only for score-attention deployments where the residual is consumed by a custom kernel.}
\label{tab:qjl_accounting}
\small
\setlength{\tabcolsep}{4pt}
\begin{tabular}{llcc|cc}
modality & nominal $b$ & bits/scalar & metric & QJL bits/scalar & QJL metric \\
\midrule
LLM (WikiText-2$\downarrow$) & 4 & 4.50 & 10.306 & 5.00 & 10.306 \\
 & 3 & 3.50 & 10.753 & 4.00 & 10.754 \\
 & 2 & 2.50 & 13.517 & 3.00 & 13.511 \\
\midrule
CausVid (LPIPS$\downarrow$) & 4 & 6.92 & 0.038 & 7.40 & 0.038 \\
 & 3 & 5.98 & 0.078 & 6.46 & 0.078 \\
 & 2 & 5.15 & 0.178 & 5.63 & 0.178 \\
\midrule
Causal Forcing (LPIPS$\downarrow$) & 4 & 5.91 & 0.309 & 6.45 & 0.310 \\
 & 3 & 4.87 & 0.390 & 5.40 & 0.389 \\
 & 2 & 3.95 & 0.581 & 4.48 & 0.581 \\
\midrule
AAR (LSD\,dB$\downarrow$) & 4 & 7.24 & 6.20 & 7.59 & 6.17 \\
 & 3 & 6.58 & 6.45 & 6.93 & 6.43 \\
 & 2 & 6.06 & 6.75 & 6.40 & 6.88 \\
\end{tabular}

\end{table}

\section{Kernel speed and KV compression}
\label{app:kernel-speed}

\Table{kernel_speed} reports wall-clock encode and decode times on a
single NVIDIA H200 ($B{=}1$, median of 50 runs after 30 warm-up
iterations) at each modality's full sequence length.
The SDPA~bf16 baseline uses pre-cast bf16 KV tensors for decode (no
per-step cast overhead); its ``encode'' column reports the one-time
fp32$\to$bf16 copy cost.

The fused decode kernels add $5$--$11{\times}$ overhead vs.\ the
highly optimised cuDNN SDPA~bf16 path, decreasing at lower bit widths
as the packed data shrinks.
This overhead is inherent: each decode step fuses
decompression---centroid lookup (TQ-MSE) or octahedral reconstruction
(\ours)---into the attention loop, trading compute for the KV memory
savings reported in the last column.
TQ-MSE encode is lightweight (${\leq}2$\,ms even at 65k tokens);
\ours encode uses a Kronecker-factored WHT and direct triad indexing,
bringing the per-token cost to ${\approx}0.08\,\mu$s ($d_h{=}128$),
which is negligible for auto-regressive LLM decode and small relative
to diffusion denoising for video.

\begin{table}[t]
\centering
\titlecaption{Encode / decode kernel speed and KV compression}{%
Per decode step on a single NVIDIA H200.
``Encode'' is the Triton compression time for $T$ key-value tokens;
``Decode'' is a single fused-attention query against $T$ compressed tokens.
SDPA bf16 encode is the bf16 KV copy baseline.
\ours encode uses the $3{\times}3$ joint rounding search
(\sect{method:rounding}).
KV ratio $=$ bf16 bytes\,/\,compressed bytes.}
\label{tab:kernel_speed}
\small
\setlength{\tabcolsep}{4pt}
\begin{tabular}{clrrrr}
$b$ & codec
  & encode (ms) & decode (ms)
  & dec\,/\,base & KV ratio \\
\midrule
\multicolumn{6}{l}{\textbf{Video} (Wan-1.3B, $H{=}16$, $d_h{=}64$, $g{=}32$, $T{=}32{,}760$)} \\[2pt]
-- & SDPA bf16 & 0.11 & 0.06 & $1.0\times$ & $1.0\times$ \\
4 & \ours & 3.8 & 0.59 & $10.4\times$ & $2.8\times$ \\
4 & TQ-MSE & 2.0 & 0.48 & $8.5\times$ & $3.0\times$ \\
3 & \ours & 3.6 & 0.50 & $8.9\times$ & $3.6\times$ \\
3 & TQ-MSE & 2.1 & 0.45 & $7.9\times$ & $3.8\times$ \\
2 & \ours & 3.6 & 0.49 & $8.6\times$ & $4.5\times$ \\
2 & TQ-MSE & 2.0 & 0.34 & $6.0\times$ & $4.9\times$ \\
\midrule
\multicolumn{6}{l}{\textbf{Audio} (AAR, $H{=}16$, $d_h{=}64$, $g{=}16$, $T{=}455$)} \\[2pt]
-- & SDPA bf16 & 0.02 & 0.03 & $1.0\times$ & $1.0\times$ \\
4 & \ours & 0.10 & 0.15 & $5.9\times$ & $2.4\times$ \\
4 & TQ-MSE & 0.19 & 0.15 & $5.6\times$ & $2.6\times$ \\
3 & \ours & 0.10 & 0.16 & $5.9\times$ & $2.9\times$ \\
3 & TQ-MSE & 0.24 & 0.15 & $5.6\times$ & $3.0\times$ \\
2 & \ours & 0.10 & 0.16 & $6.0\times$ & $3.5\times$ \\
2 & TQ-MSE & 0.19 & 0.15 & $5.8\times$ & $3.8\times$ \\
\midrule
\multicolumn{6}{l}{\textbf{LLM} (Qwen2.5-7B GQA, $H_{\mathrm{kv}}{=}4$, $d_h{=}128$, $g{=}32$, $T{=}65{,}536$)} \\[2pt]
-- & SDPA bf16 & 0.11 & 0.06 & $1.0\times$ & $1.0\times$ \\
4 & \ours & 5.2 & 0.66 & $11.3\times$ & $3.0\times$ \\
4 & TQ-MSE & 1.5 & 0.37 & $6.4\times$ & $3.1\times$ \\
3 & \ours & 4.6 & 0.55 & $9.4\times$ & $3.7\times$ \\
3 & TQ-MSE & 1.6 & 0.33 & $5.7\times$ & $3.9\times$ \\
2 & \ours & 4.5 & 0.52 & $8.9\times$ & $4.8\times$ \\
2 & TQ-MSE & 1.5 & 0.28 & $4.9\times$ & $5.1\times$ \\
\end{tabular}

\end{table}

\section{Long-context needle-in-a-haystack sweep}
\label{app:llm-niah}

\Table{llm_niah} expands the long-context retrieval summary from \sect{results:llm} across the full context-length and bit-width grid.

\begin{table}[t]
\centering
\titlecaption{Multi-key needle-in-a-haystack recall on \texttt{Qwen2.5-7B-Instruct-1M}}{RULER-style multi-key protocol: each cell plants four distractor needles plus one target needle, each carrying a fresh random 8-character alphanumeric magic value; scoring is exact-match on the target value. Recall is averaged over $5$ depth offsets $\{0.0,0.25,0.5,0.75,1.0\}$ per context length. Higher is better. Per column, the best rank is \textbf{bold} (every tied codec); the runner-up is \underline{underlined} only when the column has at least three distinct values, and a saturated column (every codec at the same recall) is left unhighlighted.}
\label{tab:llm_niah}
\small
\setlength{\tabcolsep}{5pt}
\begin{tabular}{llcccccc}
bits & codec & 4k & 8k & 16k & 32k & 64k & 128k \\
\midrule
-- & fp16 baseline & 1.00 & 1.00 & 1.00 & 1.00 & 1.00 & 1.00 \\
\midrule
4 & TurboQuant-MSE & 1.00 & \textbf{1.00} & 1.00 & 1.00 & 1.00 & 1.00 \\
4 & TurboQuant-QJL & 1.00 & \textbf{1.00} & 1.00 & 1.00 & 1.00 & 1.00 \\
4 & PolarQuant & 1.00 & \textbf{1.00} & 1.00 & 1.00 & 1.00 & 1.00 \\
4 & \textbf{\ours} & 1.00 & \textbf{1.00} & 1.00 & 1.00 & 1.00 & 1.00 \\
4 & \textbf{\ours-QJL} & 1.00 & 0.95 & 1.00 & 1.00 & 1.00 & 1.00 \\
\midrule
3 & TurboQuant-MSE & \textbf{1.00} & \textbf{1.00} & \textbf{1.00} & \textbf{1.00} & \textbf{1.00} & \textbf{1.00} \\
3 & TurboQuant-QJL & 0.80 & 0.65 & 0.75 & 0.60 & 0.55 & 0.65 \\
3 & PolarQuant & \underline{0.90} & \underline{0.75} & 0.75 & \underline{0.90} & \textbf{1.00} & \underline{0.85} \\
3 & \textbf{\ours} & \textbf{1.00} & \textbf{1.00} & \textbf{1.00} & \textbf{1.00} & \textbf{1.00} & \textbf{1.00} \\
3 & \textbf{\ours-QJL} & \textbf{1.00} & \textbf{1.00} & \textbf{1.00} & \textbf{1.00} & \textbf{1.00} & \textbf{1.00} \\
\midrule
2 & TurboQuant-MSE & \textbf{0.85} & 0.70 & \underline{0.60} & 0.50 & 0.65 & 0.55 \\
2 & TurboQuant-QJL & 0.05 & 0.00 & 0.00 & 0.00 & 0.00 & 0.00 \\
2 & PolarQuant & 0.05 & 0.05 & 0.05 & 0.05 & 0.05 & 0.00 \\
2 & \textbf{\ours} & \textbf{0.85} & \textbf{0.90} & \textbf{0.80} & \underline{0.75} & \underline{0.85} & \underline{0.70} \\
2 & \textbf{\ours-QJL} & \underline{0.80} & \underline{0.85} & \textbf{0.80} & \textbf{0.85} & \textbf{0.90} & \textbf{0.75} \\
\end{tabular}

\end{table}

\section{Memory-budget Pareto}
\label{app:mem-pareto}

\Figure{mem-pareto} recasts the LLM results as a deployment memory trade-off, making the Pareto frontier visible at fixed context length.

\begin{figure}[t]
    \centering
    \IfFileExists{images/llm_mem_pareto.pdf}{%
      \includegraphics[width=0.7\linewidth]{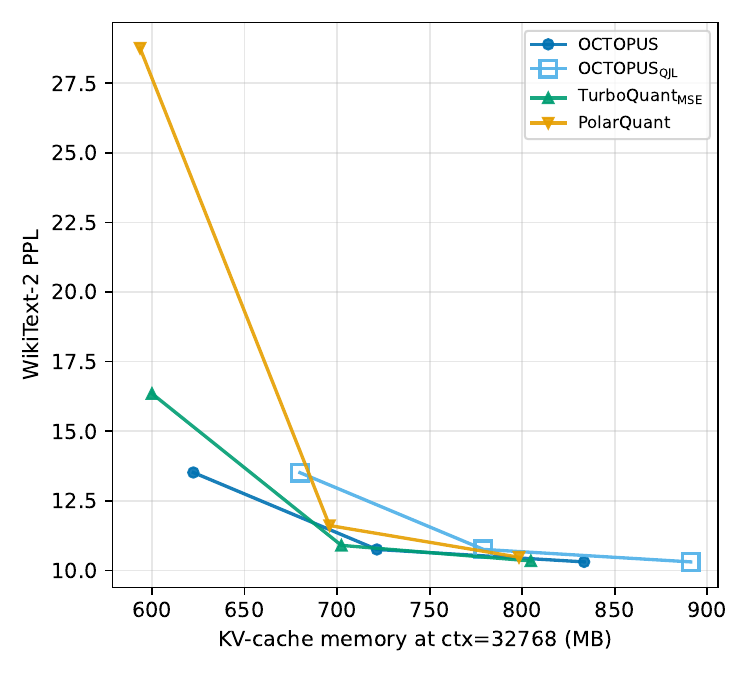}}{}
    \titlecaption{LLM quality at fixed deployment memory}{WikiText-2
      perplexity vs.\ KV-cache memory at $32{,}768$-token context for
      \texttt{Qwen2.5-7B-Instruct-1M}. The probe-time
      \texttt{kv\_cache\_bytes} is linearly extrapolated from the
      sweep's measurement window to $32$k tokens, so the x-axis is the
      memory budget a deployment actually pays. \ours dominates the
      Pareto frontier across the full memory range; \ours-QJL trails
      slightly because the $1$-bit JL residual costs a constant
      $\sim$$60$\,MB but does not move the reconstruction-quality
      curve at this context length.}%
    \label{fig:mem-pareto}
\end{figure}

\section{Full per-modality tables}
\label{app:full-tables}

\Table{video_causvid_app}, \Table{video_causal_forcing_app}, and \Table{audio_app} provide the per-codec metrics behind the modality summary in \sect{results:av}.

\begin{table}[t]
\centering
\titlecaption{CausVid video generation, expanded}{Default recipe (residual window $=1$ frame, $V$ group $g{=}32$, no per-layer boundary protection). All five codecs at the same recipe so only the codec varies. Higher is better for PSNR/SSIM/CLIP/latent-cos; lower for LPIPS. Best per bit width is \textbf{bold}, runner-up \underline{underlined}.}
\label{tab:video_causvid_app}
\small
\setlength{\tabcolsep}{3.5pt}
\begin{tabular}{llrrrrrr}
bits & codec & compr. & LPIPS $\downarrow$ & PSNR $\uparrow$ & SSIM $\uparrow$ & CLIP $\uparrow$ & lat-cos $\uparrow$ \\
\midrule
4 & TurboQuant-MSE & 2.40$\times$ & 0.0454 & 26.47 & 0.8807 & 0.3168 & 0.9894 \\
4 & TurboQuant-QJL & 2.38$\times$ & 0.0963 & 22.62 & 0.8051 & \textbf{0.3175} & 0.9680 \\
4 & PolarQuant & 2.42$\times$ & \textbf{0.0369} & \textbf{27.82} & \textbf{0.8984} & 0.3171 & \textbf{0.9922} \\
4 & \textbf{\ours} & 2.31$\times$ & \underline{0.0380} & \underline{27.52} & 0.8945 & \textbf{0.3175} & \underline{0.9920} \\
4 & \textbf{\ours-QJL} & 2.16$\times$ & \underline{0.0380} & \underline{27.52} & \underline{0.8946} & \underline{0.3174} & \textbf{0.9922} \\
\midrule
3 & TurboQuant-MSE & 2.75$\times$ & 0.0975 & 22.56 & 0.8029 & 0.3172 & 0.9674 \\
3 & TurboQuant-QJL & 2.72$\times$ & 0.2621 & 17.81 & 0.6402 & \textbf{0.3204} & 0.8707 \\
3 & PolarQuant & 2.77$\times$ & 0.0928 & 22.84 & 0.8111 & 0.3168 & \underline{0.9682} \\
3 & \textbf{\ours} & 2.67$\times$ & \textbf{0.0777} & \textbf{23.72} & \textbf{0.8298} & 0.3172 & \textbf{0.9754} \\
3 & \textbf{\ours-QJL} & 2.48$\times$ & \underline{0.0778} & \underline{23.71} & \underline{0.8297} & \underline{0.3173} & \textbf{0.9754} \\
\midrule
2 & TurboQuant-MSE & 3.22$\times$ & 0.2611 & 17.87 & 0.6429 & \textbf{0.3211} & 0.8717 \\
2 & TurboQuant-QJL & 3.19$\times$ & 0.5790 & 13.10 & 0.4503 & 0.3067 & 0.5110 \\
2 & PolarQuant & 3.26$\times$ & 0.2514 & \underline{17.93} & 0.6618 & 0.3171 & 0.8728 \\
2 & \textbf{\ours} & 3.11$\times$ & \underline{0.1784} & \textbf{19.68} & \underline{0.7190} & \underline{0.3180} & \underline{0.9225} \\
2 & \textbf{\ours-QJL} & 2.84$\times$ & \textbf{0.1783} & \textbf{19.68} & \textbf{0.7192} & 0.3179 & \textbf{0.9226} \\
\end{tabular}

\end{table}

\begin{table}[t]
\centering
\titlecaption{Causal Forcing video generation, expanded}{Same recipe and conventions as \Table{video_causvid_app}.}
\label{tab:video_causal_forcing_app}
\small
\setlength{\tabcolsep}{3.5pt}
\begin{tabular}{llrrrrrr}
bits & codec & compr. & LPIPS $\downarrow$ & PSNR $\uparrow$ & SSIM $\uparrow$ & CLIP $\uparrow$ & lat-cos $\uparrow$ \\
\midrule
4 & TurboQuant-MSE & 2.84$\times$ & 0.3339 & 14.58 & 0.5547 & \underline{0.3163} & 0.8051 \\
4 & TurboQuant-QJL & 2.81$\times$ & 0.4213 & 13.09 & 0.4884 & \textbf{0.3169} & 0.7374 \\
4 & PolarQuant & 2.87$\times$ & \textbf{0.3009} & \textbf{15.41} & \textbf{0.5850} & 0.3156 & \textbf{0.8275} \\
4 & \textbf{\ours} & 2.71$\times$ & \underline{0.3093} & 15.21 & 0.5757 & 0.3152 & \underline{0.8192} \\
4 & \textbf{\ours-QJL} & 2.48$\times$ & 0.3103 & \underline{15.23} & \underline{0.5775} & 0.3162 & 0.8190 \\
\midrule
3 & TurboQuant-MSE & 3.41$\times$ & 0.4225 & 13.10 & 0.4861 & \textbf{0.3174} & 0.7347 \\
3 & TurboQuant-QJL & 3.37$\times$ & 0.7786 & 8.42 & 0.1664 & 0.1513 & 0.2497 \\
3 & PolarQuant & 3.46$\times$ & 0.4018 & 13.06 & 0.4971 & 0.3144 & 0.7471 \\
3 & \textbf{\ours} & 3.29$\times$ & \underline{0.3904} & \underline{13.48} & \underline{0.5090} & 0.3157 & \underline{0.7594} \\
3 & \textbf{\ours-QJL} & 2.96$\times$ & \textbf{0.3892} & \textbf{13.54} & \textbf{0.5123} & \underline{0.3158} & \textbf{0.7627} \\
\midrule
2 & TurboQuant-MSE & 4.28$\times$ & 0.7770 & 8.45 & 0.1664 & 0.1499 & 0.2508 \\
2 & TurboQuant-QJL & 4.21$\times$ & 0.8164 & 7.10 & 0.1106 & 0.0845 & 0.1481 \\
2 & PolarQuant & 4.35$\times$ & \underline{0.7273} & 8.62 & 0.2155 & 0.2089 & 0.2984 \\
2 & \textbf{\ours} & 4.05$\times$ & \textbf{0.5808} & \textbf{10.90} & \textbf{0.3593} & \underline{0.3024} & \textbf{0.5618} \\
2 & \textbf{\ours-QJL} & 3.57$\times$ & \textbf{0.5808} & \underline{10.89} & \underline{0.3577} & \textbf{0.3034} & \underline{0.5592} \\
\end{tabular}

\end{table}

\begin{table}[t]
\centering
\titlecaption{Autoregressive audio (AAR), expanded}{Metrics averaged over $100$ random $10$\,s AudioSet-20k clips used as CLAP-audio conditioning at the default recipe (residual window $1$ scale, $V$ group $g{=}16$, no per-layer protection). Best per bit width is \textbf{bold}, runner-up \underline{underlined}.}
\label{tab:audio_app}
\small
\setlength{\tabcolsep}{3.5pt}
\begin{tabular}{llrrrrr}
bits & codec & compr. & LSD $\downarrow$ & mel-MSE $\downarrow$ & SNR $\uparrow$ & lat-cos $\uparrow$ \\
\midrule
4 & TurboQuant-MSE & 2.29$\times$ & 6.36 & 0.2191 & 2.07 & -- \\
4 & TurboQuant-QJL & 2.26$\times$ & 6.24 & 0.2284 & 1.80 & -- \\
4 & PolarQuant & 2.31$\times$ & 6.28 & \underline{0.2130} & \textbf{2.23} & -- \\
4 & \textbf{\ours} & 2.21$\times$ & \underline{6.20} & 0.2134 & \underline{2.22} & -- \\
4 & \textbf{\ours-QJL} & 2.11$\times$ & \textbf{6.17} & \textbf{0.2104} & 2.19 & -- \\
\midrule
3 & TurboQuant-MSE & 2.48$\times$ & 6.48 & 0.2377 & 1.55 & -- \\
3 & TurboQuant-QJL & 2.46$\times$ & 12.72 & 1.4883 & -5.44 & -- \\
3 & PolarQuant & 2.51$\times$ & \textbf{6.34} & \textbf{0.2279} & \textbf{1.71} & -- \\
3 & \textbf{\ours} & 2.43$\times$ & 6.45 & \underline{0.2343} & 1.51 & -- \\
3 & \textbf{\ours-QJL} & 2.31$\times$ & \underline{6.43} & 0.2346 & \underline{1.60} & -- \\
\midrule
2 & TurboQuant-MSE & 2.72$\times$ & 12.65 & 1.4547 & -5.30 & -- \\
2 & TurboQuant-QJL & 2.69$\times$ & 13.17 & 1.6676 & -5.98 & -- \\
2 & PolarQuant & 2.75$\times$ & 12.59 & 1.4325 & -5.28 & -- \\
2 & \textbf{\ours} & 2.64$\times$ & \textbf{6.75} & \textbf{0.3171} & \textbf{1.07} & -- \\
2 & \textbf{\ours-QJL} & 2.50$\times$ & \underline{6.88} & \underline{0.3240} & \underline{0.97} & -- \\
\end{tabular}

\end{table}

\section{Stills}
\label{app:stills}

\Figure{causvid_stills} shows representative worst-case frames for the video codecs, complementing the aggregate LPIPS/PSNR/SSIM numbers in \sect{results:av}.

\begin{figure}[htb]
    \centering
    \begin{subfigure}[b]{\linewidth}
        \includegraphics[width=0.9\linewidth]{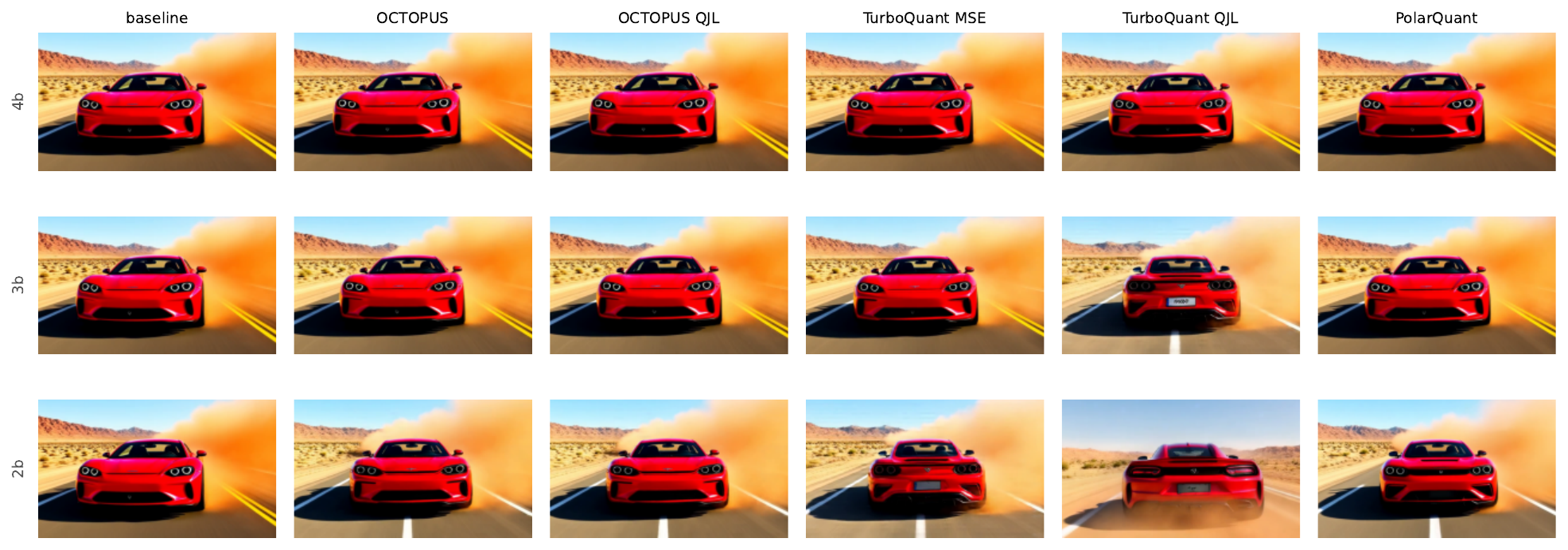}
        \caption{CausVid}
    \end{subfigure}\\[0.5ex]
    \begin{subfigure}[b]{\linewidth}
        \includegraphics[width=0.9\linewidth]{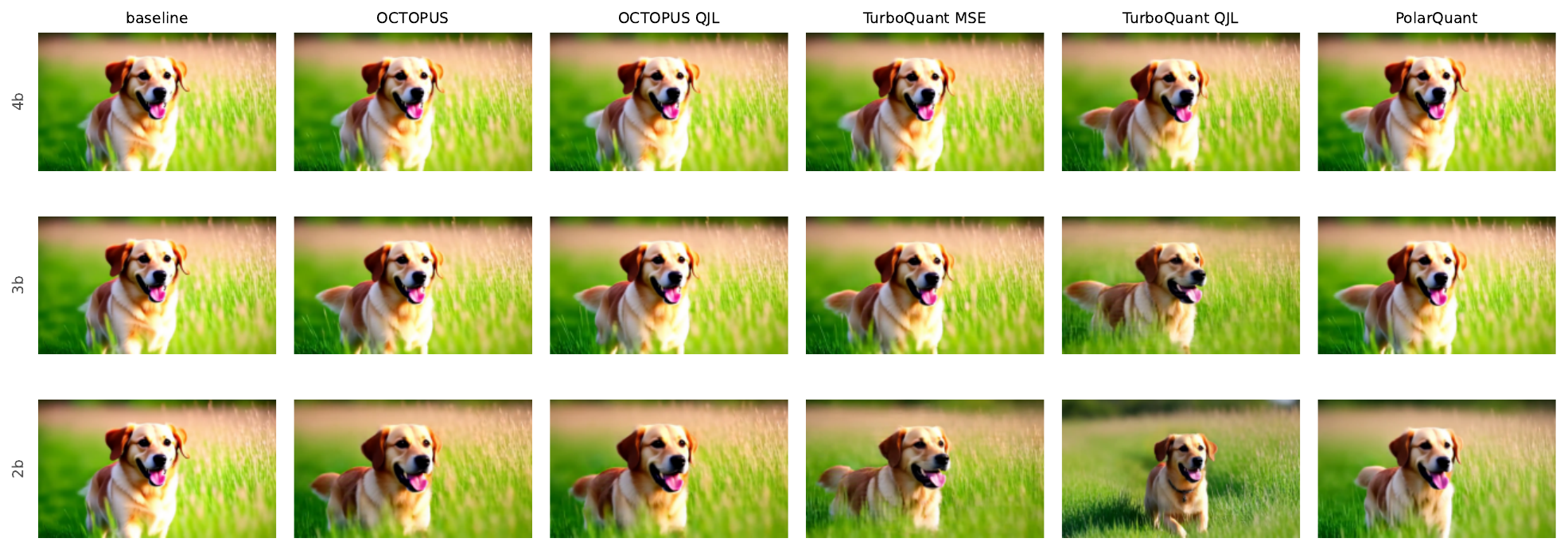}
        \caption{CausVid}
    \end{subfigure}\\[0.5ex]
    \begin{subfigure}[b]{\linewidth}
        \includegraphics[width=0.9\linewidth]{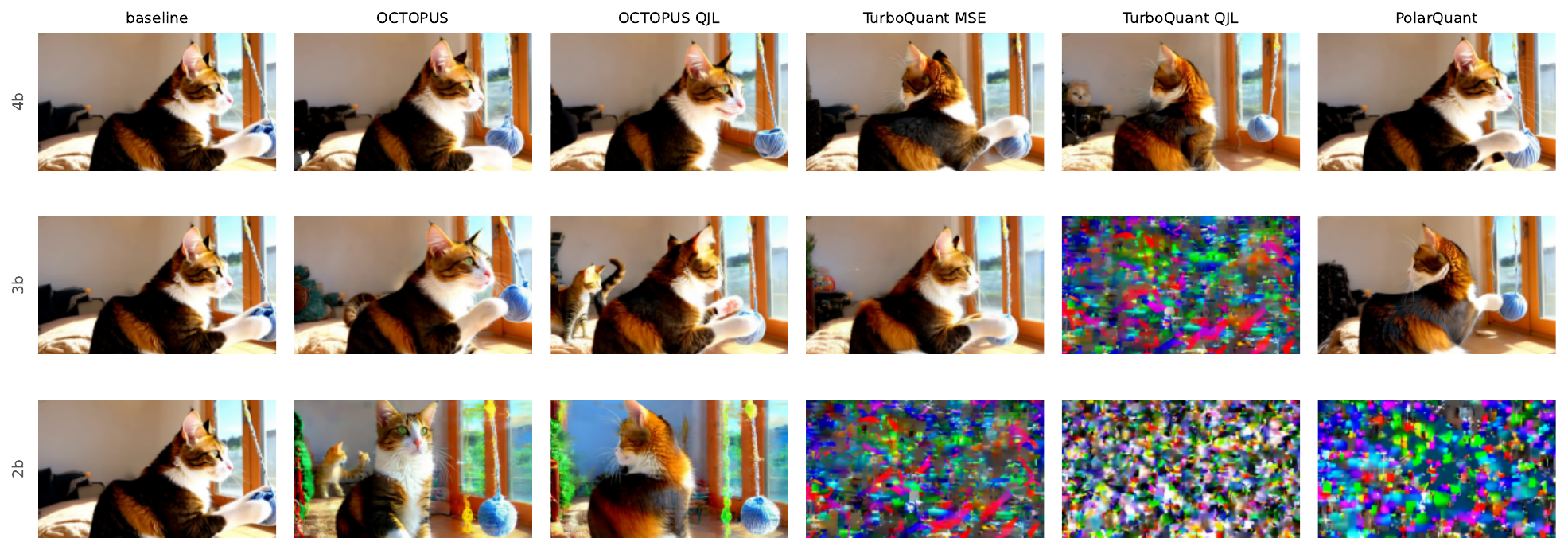}
        \caption{Causal Forcing}
    \end{subfigure}\\[0.5ex]
    \begin{subfigure}[b]{\linewidth}
        \includegraphics[width=0.9\linewidth]{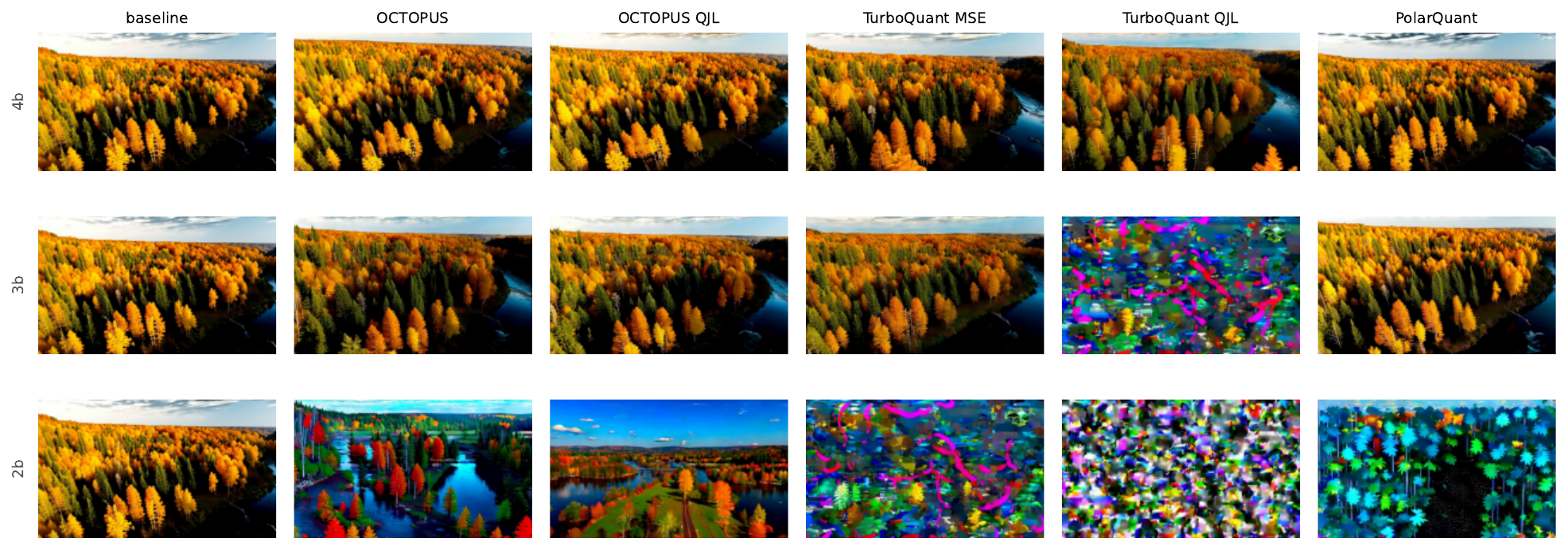}
        \caption{Causal Forcing}
    \end{subfigure}
    \titlecaption{Worst-case codec divergence across bit depths}{%
      Each panel shows the single frame with the highest combined
      cross-codec L1 divergence from the fp16 baseline (same frame index
      for both pipelines). Rows: $b{=}4,3,2$; columns: baseline and each
      codec. \ours remains visually faithful at every bit width; competing
      codecs collapse at $b{\le}3$.}%
    \label{fig:causvid_stills}
\end{figure}

\ifshowchecklist
  \clearpage
  \input{checklist.tex}
\fi

\end{document}